\documentclass{article}
\pdfoutput=1
\usepackage{nips15submit_e, times}
\usepackage{graphicx} %
\usepackage[center]{subfigure} 
\usepackage[round]{natbib}
\usepackage{multibib}
\usepackage{algorithmic}
\usepackage[ruled]{algorithm2e}
\SetAlFnt{\small}
\SetAlCapFnt{\small}
\SetAlCapNameFnt{\small}
\algsetup{linenosize=\small}
\usepackage{enumitem}
\usepackage{url}
\usepackage{cancel}
\usepackage{amsmath,amsfonts} %
\usepackage[bookmarks=false]{hyperref}
\usepackage{amsthm}
\usepackage{comment}
\usepackage{tikz-qtree}
\usepackage{listings}
\usepackage{bbm}
\usepackage{mathtools}
\usepackage{breqn}
\usepackage{todonotes}
\usepackage{rotating}
\usepackage{ amssymb }
\usepackage{booktabs}
\usepackage{mathtools}
\usepackage{bm}
\usepackage{mathdots}
\newcites{sup}{Supplementary References}

\DeclarePairedDelimiter\brangle{\langle}{\rangle}
\newtheorem{theorem}{Theorem}[section]
\newtheorem{lemma}[theorem]{Lemma}

\newtheorem{definition}[theorem]{Definition}
\newtheorem{properties}[theorem]{Property}
\newtheorem{problem}[theorem]{Problem}
\hypersetup{
  plainpages=false,
  colorlinks=true,   %
  linkcolor=blue,    %
  anchorcolor=blue,  %
  citecolor=blue,    %
  filecolor=blue,    %
  pagecolor=blue,    %
  urlcolor=blue,     %
  pdfview=FitH,               %
  pdfstartview=FitH,          %
  pdfpagelayout=SinglePage    %
}

\let\oldnl\nl%
\newcommand{\nonl}{\renewcommand{\nl}{\let\nl\oldnl}}%

\theoremstyle{remark}
\newtheorem{remark}{Remark}

\newcommand{\shrinkAmount}{\delta}

\DeclareMathOperator*{\conv}{conv}
\DeclareMathOperator{\diam}{diam}

\DeclareMathOperator*{\argmin}{\arg\min}
\DeclareMathOperator*{\argmax}{\arg\max}

\newcommand{\stepsize}{\gamma}
\newcommand{\stepmax}{{\stepsize}_{\mathrm{max}}}
\newcommand{\FW}{{\hspace{0.05em}\textnormal{FW}}}
\newcommand{\pFW}{{\hspace{0.05em}\textnormal{pFW}}}
\newcommand{\away}{{\hspace{0.06em}\textnormal{\scriptsize A}}}
\newcommand{\Cf}{C_{\hspace{-0.08em}f}}

\newcommand{\x}{\bm{x}}

\newcommand{\s}{\bm{s}}
\newcommand{\dd}{\bm{d}}

\newcommand{\vv}{\bm{v}} %

\newcommand{\Vertices}{\mathcal{V}}
\newcommand{\Coreset}{\mathcal{S}}

\newcommand{\innerProd}[2]{\left\langle #1 , #2 \right\rangle}

\newcommand{\LOCAL}{\mathbb{L}} %
\newcommand{\MARG}{\mathcal{M}} %
\newcommand{\VAL}{\textrm{VAL}}

\newcommand{\xk}{\bm{x}^{(k)}}
\newcommand{\xnext}{\bm{x}^{(k+1)}}

\newcommand{\sk}{\bm{s}^{(k)}}
\newcommand{\snext}{\bm{s}^{(k+1)}}

\newcommand{\Cconst}{\widetilde{C}}

\newcommand{\fk}{f(\bm{x}^{(k)})}

\newcommand{\hk}{h_k}
\newcommand{\hkarg}[1]{h_{#1}} %
\newcommand{\hnext}{h_{k+1}}
\newcommand{\hprev}{h_{k-1}}
\newcommand{\gk}{g(\bm{x}^{(k)})}

\newcommand{\DOM}{\mathcal{D}}
\newcommand{\xopt}{\bm{x}^{*}}
\newcommand{\TREEPOL}{\mathbb{T}} %

\newcommand{\modContinuity}{\omega} %

\newcommand{\unif}{\bm{u}_0} %
\newcommand{\gunif}{g_u(\bm{x}^{(k)})} %
\newcommand{\gunifNox}{g_u} %
\newcommand{\geps}{g_{(\shrinkAmount)}(\bm{x}^{(k)})}
\newcommand{\gepsk}[1]{g_{(\shrinkAmount^{(#1)})}(\bm{x}^{(k)})}
\newcommand{\Meps}{\mathcal{M}_{\shrinkAmount}}
\newcommand{\DOMeps}{\mathcal{D}_{\shrinkAmount}}
\newcommand{\dkeps}{\bm{d}^{(k)}_{(\shrinkAmount)}}
\newcommand{\skeps}{\bm{s}^{(k)}_{(\shrinkAmount)}}

\newcommand{\xopteps}{\bm{x}^{*}_{(\shrinkAmount)}}
\newcommand{\xoptshift}{\tilde{\bm{x}}_{(\shrinkAmount)}}

\newcommand{\epsk}[1]{\shrinkAmount^{(#1)}}
\newcommand{\del}{\mathrm{\bm{d}}}
\newcommand{\h}{h}
\newcommand{\vrho}{\bm{\rho}}
\newcommand{\vmu}{\vec{\bm{\mu}}}
\newcommand{\bmu}{\bm{\mu}}
\newcommand{\edges}{\{i,j\}\in E}

\newcommand{\vtheta}{\vec{\bm{\theta}}}
\newcommand{\btheta}{\bm{\theta}}

\newcommand{\trwent}{H(\vec{\bm{\mu}};\vrho)}
\newcommand{\LOCALSEARCH}{\textbf{\textrm{LOCALSEARCH}}}
\newcommand{\CORRECTION}{\textbf{\textrm{CORRECTION}}}
\newcommand{\MIfunction}{\textbf{\textrm{edgesMI}}}
\newcommand{\MAXITS}{\textbf{\textrm{MAXITS}}}
\newcommand{\MAXRHOITS}{\textbf{\textrm{MAX$\_$RHO$\_$ITS}}}
\newcommand{\TRW}{\text{TRW}(\vmu;\vtheta,\vrho)}
\newcommand{\algComment}[1]{\quad\emph{\small{(#1)}}}
\newcommand{\nbrs}[1]{\mathcal{N}(#1)}
\newcommand{\gunifBound}{B}
\newcommand{\stopCrit}{\epsilon}

\title{Barrier Frank-Wolfe for Marginal Inference}

\author{
Rahul G. Krishnan \\
Courant Institute\\
New York University\\
\And
Simon Lacoste-Julien \\ 
INRIA - Sierra Project-Team\\
\'{E}cole Normale Sup\'{e}rieure, Paris
\And
David Sontag\\
Courant Institute\\
New York University\\
}

\nipsfinalcopy %

\begin{document}
\maketitle

\begin{abstract} 
We introduce a globally-convergent algorithm for optimizing the
tree-reweighted (TRW) variational objective over the marginal
polytope.
The algorithm is based on the conditional gradient method (Frank-Wolfe) 
and moves pseudomarginals within the marginal
polytope through repeated maximum a posteriori (MAP) calls.
This modular structure enables us 
to leverage black-box MAP solvers (both exact and approximate) for
variational inference, and obtains more accurate results than
tree-reweighted algorithms that optimize over the local consistency
relaxation.
Theoretically, we bound the sub-optimality for the proposed algorithm
despite the TRW objective having unbounded gradients at the boundary of the marginal polytope. 
Empirically, we demonstrate the increased quality of results found  
by tightening the relaxation over the marginal polytope as well 
as the spanning tree polytope on synthetic and real-world instances. 
\end{abstract} 
\section{Introduction}
\vspace{-3mm}

Markov random fields (MRFs) are used in many areas of computer science such as vision and speech. 
Inference in these undirected graphical models is generally intractable. 
Our work focuses on performing approximate marginal inference 
by optimizing the Tree Re-Weighted (TRW)
objective \citep{wainwright2005new}. 
The TRW objective is concave, is exact for tree-structured MRFs, 
and provides an upper bound on the log-partition function. 

Fast combinatorial solvers for the TRW objective exist, including Tree-Reweighted 
Belief Propagation (TRBP) \citep{wainwright2005new}, convergent
message-passing based on geometric programming
\citep{globerson2012convergent}, and dual decomposition
\citep{Jancsary11}. These methods optimize over the set
of pairwise consistency constraints, also called the local polytope. \citet{sontag2007new} showed that significantly better 
results could be obtained by optimizing over tighter relaxations of the marginal 
polytope. However, deriving a message-passing algorithm for the TRW objective over tighter relaxations of the marginal polytope
is challenging. Instead, \citet{sontag2007new} use the
conditional gradient method (also called Frank-Wolfe) and
off-the-shelf linear programming solvers to optimize TRW over the cycle
consistency relaxation.
Rather than optimizing over the cycle relaxation,
\citet{belangerWorkshop2013} optimize the TRW objective over
the exact marginal polytope. Then,
using Frank-Wolfe, the linear minimization performed in
the inner loop can be shown to correspond to MAP inference.

The Frank-Wolfe optimization algorithm has seen increasing use in machine learning,
thanks in part to its efficient handling of complex constraint sets appearing with structured data~\citep{jaggi2013revisiting,lacoste2015MFW}.
However, applying Frank-Wolfe to
variational inference presents challenges that were never
resolved in previous work. 
First, the linear minimization performed in the
inner loop is computationally expensive, either requiring repeatedly solving a
large linear program, as in \citet{sontag2007new}, or performing
MAP inference, as in \citet{belangerWorkshop2013}.
Second, the TRW objective involves entropy terms whose gradients go to
infinity near the boundary of the feasible set, therefore
existing convergence guarantees for Frank-Wolfe do not apply.
Third, variational inference using TRW involves
both an outer and inner loop of Frank-Wolfe, where the outer loop
optimizes the edge appearance probabilities in the TRW entropy bound to tighten it. 
Neither \citet{sontag2007new} nor \citet{belangerWorkshop2013} explore the effect of
optimizing over the edge appearance probabilities. %
Although MAP inference is in general NP hard \citep{MAP_NP_Hard},
it is often possible to find exact solutions to large real-world
instances within reasonable running times 
\citep{SontagEtAl_uai08, allouche2010toulbar2,kappes2013comparative}. 
Moreover, as we show in our experiments, even approximate MAP solvers
can be successfully used within our variational inference algorithm.
As MAP solvers improve in their
runtime and performance, their iterative use could become feasible and as a byproduct enable more efficient and accurate
marginal inference. Our work provides a fast deterministic alternative to
recently proposed Perturb-and-MAP algorithms
\citep{papandreou2011perturb, hazan2012partition, ermonWISH}.

\textbf{Contributions.} 
This paper makes several theoretical and practical innovations. 
We propose a modification to the Frank-Wolfe algorithm that 
optimizes over adaptively chosen contractions of the domain and prove its 
rate of convergence for functions
whose gradients can be unbounded at the boundary.  
Our algorithm does not require a different oracle than standard Frank-Wolfe and could be 
useful for other convex optimization problems where the gradient is ill-behaved at the boundary.

We instantiate the algorithm for approximate marginal inference over the marginal polytope 
with the TRW objective.
With an exact MAP oracle, we obtain the first provably convergent algorithm 
for the optimization of the TRW objective over the marginal polytope, 
which had remained an open problem to the best of our knowledge.
Traditional proof techniques of convergence for first order methods fail as the gradient of the TRW objective is not Lipschitz continuous. 

We develop several heuristics to
make the algorithm practical: a fully-corrective
variant of Frank-Wolfe that reuses previously found integer
assignments thereby reducing the need for new (approximate) MAP
calls, the use of local search between MAP calls, 
and significant re-use of computations between
subsequent steps of optimizing over the spanning tree
polytope.
We perform an extensive experimental evaluation on both synthetic and real-world inference tasks. 

\vspace{-2mm}
\section{Background}
\vspace{-3mm}

{\bf Markov Random Fields}:
MRFs are undirected probabilistic graphical models 
where the probability distribution factorizes over cliques in the graph.
We consider marginal inference on pairwise MRFs with $N$ random variables $X_1,X_2,\ldots,X_N$
where each variable takes discrete states $x_i\in \VAL_i$. Let $G=(V,E)$ be the Markov 
network with an undirected edge $\edges$ for every two 
variables $X_i$ and $X_j$ that are connected together. Let $\nbrs{i}$ refer to the 
set of neighbors of variable $X_i$. We organize the edge log-potentials $\theta_{ij}(x_i, x_j)$
for all possible values of $x_i \in \VAL_i$, $x_j \in \VAL_j$ in the vector $\btheta_{ij}$,
and similarly for the node log-potential vector $\btheta_i$. We regroup these in the overall
vector~$\vtheta$. We introduce a similar grouping for the marginal vector $\vmu$:
for example, $\mu_i(x_i)$ gives the coordinate of the marginal vector corresponding
to the assignment $x_i$ to variable $X_i$.

{\bf Tree Re-weighted Objective} \citep{wainwright2005new}:
Let $Z(\vtheta)$ be the partition function for the MRF and $\MARG$ be the set of all valid marginal vectors (the marginal polytope). The maximization of the TRW objective gives the following
upper bound on the log partition function:
\begingroup 
\vspace{-1mm}
\setlength\belowdisplayskip{-12pt} %
\begin{equation}
\begin{split}
\label{eqn:upperBoundPartition}
\log Z(\vtheta) &\leq \min_{\vrho \in \TREEPOL} \max_{\vmu \in \MARG}\quad \underbrace{\brangle{\vtheta,\vmu} + \trwent}_{\TRW},
\end{split}
\end{equation}
where the TRW entropy is:
\endgroup
\begin{equation}
	\label{eqn:trw_entropy_like_bethe}
	\trwent := \sum_{i\in V}(1-\hspace{-3mm}\sum_{j\in \nbrs{i}}\hspace{-2mm}\rho_{ij})H(\bmu_i) 
	+\hspace{-2mm} \sum_{(ij)\in
          E}\hspace{-1mm}\rho_{ij}H(\bmu_{ij}), \quad H(\bmu_i) :=   -\sum_{x_i} \mu_i(x_i)\log\mu_i(x_i).
\end{equation}
$\TREEPOL$ is the spanning tree polytope, the convex hull
of edge indicator vectors of all possible spanning trees of the graph. 
Elements of $\vrho\in \TREEPOL$ specify the probability of an
edge being present under a specific distribution over spanning trees.
$\MARG$ is difficult to optimize over, and most
TRW algorithms optimize over a relaxation called
the local consistency polytope $\LOCAL \supseteq \MARG$:
\vspace{-1mm}
\begin{equation*}
\resizebox{\hsize}{!}{$\LOCAL:=\left\{\vmu\geq \bm{0},\;\sum_{x_i} \mu_i(x_i) = 1 \;\forall i\in V,\;\sum_{x_i} \mu_{ij}(x_i,x_j) = \mu_j(x_j), \sum_{x_j} \mu_{ij}(x_i,x_j) = \mu_i(x_i)\;\;\forall \edges \right\}.$}
\end{equation*}
\vspace{-3mm}

The TRW objective $\TRW$ is a globally concave function of $\vmu$ over $\LOCAL$,
assuming that $\vrho$ is obtained from a valid distribution over spanning
trees of the graph (i.e. $\vrho \in \TREEPOL$). 

{\bf Frank-Wolfe (FW) Algorithm:} 
In recent years, the Frank-Wolfe (aka conditional gradient) algorithm has gained 
popularity in machine learning \citep{jaggi2013revisiting} for the 
optimization of convex functions over compact domains (denoted $\DOM$). The algorithm is used to 
solve 
$\min_{\x\in \DOM} f(\x)$	
by iteratively finding a good descent vertex by solving the linear subproblem:
\begin{equation}
\label{eqn:lin_subproblem}
\sk = \argmin_{\s\in\mathcal{D}}\brangle{\nabla f(\xk), \s} \qquad \text{(FW oracle)},
\end{equation}
and then taking a convex step towards this vertex: $\x^{(k+1)} = (1-\stepsize) \x^{(k)} + \stepsize \sk$ for
a suitably chosen step-size $\stepsize \in [0,1]$.
The algorithm remains within the feasible set (is projection free), is invariant to affine transformations of the domain, and can be implemented in a memory efficient manner. Moreover, the FW gap $g(\x^{(k)}) := \brangle{-\nabla f(\xk), \sk - \x^{(k)}}$ provides an upper bound on the suboptimality of the iterate $\x^{(k)}$.
The primal convergence of the Frank-Wolfe algorithm is given by 
Thm.~$1$ in \citet{jaggi2013revisiting}, restated here for convenience: for $k\geq1$, the iterates $\xk$ satisfy:
\vspace{-3mm}
\begin{dmath}
	\label{eqn:fw_theorem_convergence_original}
f(\xk) - f(\xopt) \leq \frac{2\Cf}{k+2} ,
\end{dmath}
\vspace{-3mm}
where $\Cf$ is called the ``curvature constant''. Under the assumption 
that $\nabla f$ is $L$-Lipschitz continuous\footnote{I.e. $\|\nabla{f}(\x)-\nabla{f}(\x')\|_* \leq L \|\x-\x'\|$ for $\x, \x' \in \DOM$. Notice that the dual norm $\| \cdot \|_*$ is needed here.} on $\DOM$, we can bound it as
$\Cf \leq L \diam_{||.||}(\mathcal{D})^2$.

\textbf{Marginal Inference with Frank-Wolfe:}
To optimize $\max_{\vmu\in\MARG} \TRW$ with Frank-Wolfe,
the linear subproblem~\eqref{eqn:lin_subproblem} becomes $\argmax_{\vmu\in\mathcal{M}}\brangle{\tilde{\btheta},\vmu}$, where the perturbed potentials $\tilde{\btheta}$ 
correspond to the gradient of $\TRW$ with respect to $\vmu$. %
Elements of $\tilde{\btheta}$ are of the 
form $\theta_c(x_c) + K_c (1+\log\mu_c(x_c))$, evaluated at the pseudomarginals'
current location in $\mathcal{M}$, where 
$K_c$ is the coefficient of the entropy for the node/edge
term in~\eqref{eqn:trw_entropy_like_bethe}.
The FW linear subproblem here is thus equivalent to performing MAP inference in a graphical model with
potentials $\tilde{\btheta}$ \citep{belangerWorkshop2013}, as
the vertices of the marginal polytope are in 1-1 
correspondence with valid joint assignments to the random variables of the MRF,
and the solution of a linear program is always achieved at a
vertex of the polytope. 
The TRW objective does not have a Lipschitz continuous gradient over $\MARG$, and so standard convergence proofs for Frank-Wolfe do not hold.

\vspace{-2mm}
\section{Optimizing over Contractions of the Marginal Polytope}
\vspace{-3mm}

\textbf{Motivation}: 
We wish to (1) use the fewest possible MAP calls, and (2)
avoid regions near the boundary where the unbounded curvature of the function
slows down convergence.  
A viable option to address (1) is through the use of \emph{correction steps}, where after a Frank-Wolfe step, 
one optimizes
over the polytope defined by previously visited vertices of $\MARG$ (called the fully-corrective
Frank-Wolfe (FCFW) algorithm and proven to be linearly convergence for strongly convex objectives~\citep{lacoste2015MFW}). 
This does not require additional %
MAP calls. 
However, we found (see Sec.~\ref{sec:M_M_eps}) that when optimizing the TRW objective over $\MARG$, performing correction steps
can surprisingly \emph{hurt} performance.
This leaves us in a dilemma: correction steps 
enable decreasing the objective without additional MAP calls,
but they can also slow global progress since iterates after correction
sometimes lie close to the boundary of the polytope (where the FW directions become less informative). 
In a manner akin to barrier methods and to \citet{garber2013linearly}'s
local linear oracle, our proposed solution maintains the iterates within a contraction of the
polytope. This gives us most of the mileage obtained from performing the
correction steps \emph{without} suffering the consequences of venturing too close to the
boundary of the polytope. We prove a global convergence rate for the
iterates with respect to the true solution over the full polytope.

We describe convergent algorithms to optimize $\TRW$ for $\vmu \in \MARG$. 
The approach we adopt to deal with the issue of unbounded gradients at the boundary
is to perform Frank-Wolfe within a contraction of the marginal polytope
given by $\Meps$ for $\shrinkAmount \in [0,1]$, with either a fixed $\shrinkAmount$ or an adaptive~$\shrinkAmount$.
\begin{definition}[Contraction polytope]
	$\Meps := (1-\shrinkAmount) \MARG + \shrinkAmount\, \unif$,
        where $\unif\in \MARG$ is the vector representing the uniform
        distribution.
\end{definition}
\vspace{-1mm}
Marginal vectors that lie within $\Meps$ are bounded away from zero as all the components of $\unif$
are strictly positive. 
Denoting $\Vertices^{(\shrinkAmount)}$ as the
set of vertices of $\Meps$, $\Vertices$ as the set of vertices of $\MARG$ and $f(\vmu) := -\TRW$, the key insight
that enables our novel approach is that:
\small
\begin{equation*}
	\underbrace{\argmin_{\vv^{(\shrinkAmount)}\in \Vertices^{(\shrinkAmount)}}\innerProd{\nabla f}{\vv^{(\shrinkAmount)}}}_{\algComment{Linear Minimization over $\Meps$}} \equiv\argmin_{\vv \in \Vertices} \underbrace{\innerProd{\nabla f}{(1-\shrinkAmount)\vv + \shrinkAmount\unif}}_{\algComment{Definition of $\vv^{(\shrinkAmount)}$}} 
\equiv\underbrace{(1-\shrinkAmount) \argmin_{\vv \in \Vertices}\innerProd{\nabla f}{\vv} + \shrinkAmount\unif.}_{\algComment{Run MAP solver and shift vertex}} 
\end{equation*}
\normalsize
Therefore, to solve the FW subproblem~\eqref{eqn:lin_subproblem} over $\Meps$, we can run as usual a MAP solver  and simply shift the resulting vertex of $\MARG$ towards $\unif$ to obtain a vertex of $\Meps$. 
Our solution to optimize over restrictions of the polytope is more broadly applicable to the optimization problem defined below, with $f$ satisfying Prop.~\ref{prop:prop_bounded_grad} (satisfied by the TRW objective) in order to get convergence rates.
\begin{problem} 	
\label{prop:generic_fxn}
Solve $\min_{\x\in\DOM}f(\x)$ where $\DOM$ is a compact convex set and $f$ is convex and continuously differentiable on the relative interior of $\DOM$.
\end{problem}
\begin{properties}
\hyperref[sec:trw_bounded_lip]{(Controlled growth of Lipschitz constant over $\DOMeps$)}.
\label{prop:prop_bounded_grad}
We define $\DOMeps:= (1-\shrinkAmount)\DOM + \shrinkAmount\unif$ for a fixed $\unif$ in the relative interior of $\DOM$. We suppose that there exists a fixed $p \geq 0$ and $L$ such that for any $\shrinkAmount > 0$, $\nabla f(\x)$ has a bounded Lipschitz constant $L_{\shrinkAmount} \leq L\shrinkAmount^{-p}\,\,\,\forall \x\in\DOMeps$.
\end{properties}

\textbf{Fixed $\shrinkAmount$:}
The first algorithm fixes a value for $\shrinkAmount$ a-priori
and performs the optimization over $\DOMeps$. 
The following theorem bounds the 
sub-optimality of the iterates with respect to the optimum over $\DOM$. 

\begin{theorem}[Suboptimality bound for fixed-$\shrinkAmount$ algorithm]
	\label{thm:convergence_fixed_eps_main}
	Let $f$ satisfy the properties in Prob.~\ref{prop:generic_fxn} and Prop.~\ref{prop:prop_bounded_grad}, and suppose further that $f$ is finite on the boundary of $\DOM$. 
	Then the use of Frank-Wolfe for $\min_{\x\in\DOMeps} f(\x)$ realizes a sub-optimality over $\DOM$ bounded as:
	$$f(\xk)-f(\xopt)\leq
        \frac{2C_{\shrinkAmount}}{(k+2)}+\modContinuity \left( \shrinkAmount \diam(\DOM)\right),$$ 
        where $\xopt$ is the optimal
        solution in $\DOM$, $C_{\shrinkAmount} \leq 
        L_{\shrinkAmount} \diam_{||.||}(\DOMeps)^2 $, and $\modContinuity$
        is the modulus of continuity function of the (uniformly) continuous $f$ (in particular, $\modContinuity(\delta) \downarrow 0$ as $\delta \downarrow 0$).
\end{theorem}
The full proof is given in App.~\ref{sec:theory_fixed_eps}.
The first term of the bound
comes from the standard Frank-Wolfe convergence analysis of the
sub-optimality of $\xk$ relative to $\xopteps$, the optimum over $\DOMeps$, as in~\eqref{eqn:fw_theorem_convergence_original} and using Prop.~\ref{prop:prop_bounded_grad}.
The second term arises by bounding $f(\xopteps) - f(\x^*) \leq f(\tilde{\x}) - f(\x^*)$ with 
a cleverly chosen $\tilde{\x} \in \DOMeps$ (as $\xopteps$ is optimal in $\DOMeps$). We pick 
$\tilde{\x} := (1-\shrinkAmount) \x^* + \shrinkAmount \unif$ and note that $\|\tilde{\x} - \x^*\| \leq \shrinkAmount \diam(\DOM)$. As $f$ is continuous on a compact set, it is uniformly continuous 
and we thus have $f(\tilde{\x}) - f(\x^*) \leq \modContinuity (\shrinkAmount \diam(\DOM))$ with $\modContinuity$ its modulus of continuity function.

\textbf{Adaptive $\shrinkAmount$: }
The second variant to solve $\min_{\x\in\DOM} f(\x)$ iteratively perform FW steps over $\DOMeps$, but also decreases $\shrinkAmount$ adaptively.
The update schedule for $\shrinkAmount$
is given in Alg.~\ref{alg:adaptive_update} and is motivated by the convergence proof.
The idea is to ensure that the FW gap over $\DOMeps$ is always at least half the FW gap over $\DOM$,
relating the progress over $\DOMeps$ with the one over $\DOM$. 
It turns out that $\text{FW-gap-}\DOMeps = (1-\shrinkAmount) \text{FW-gap-}\DOM + \shrinkAmount\cdot\gunif$, where the ``uniform gap'' $\gunif$ quantifies the decrease of the function when contracting towards $\unif$.
When $\gunif$ is negative and large compared to 
the FW gap, we need to shrink $\shrinkAmount$ (see step~5 in~Alg.~\ref{alg:adaptive_update}) to
ensure that the $\shrinkAmount$-modified direction 
is a sufficient descent direction. 
\begin{algorithm}[t]
	\caption{Updates to $\shrinkAmount$ after a MAP call (Adaptive $\shrinkAmount$ variant)} 
	\label{alg:adaptive_update}
	\begin{algorithmic}[1]
	\STATE At iteration $k$. Assuming $\xk,\unif,\epsk{k-1},f$ are defined and $\sk$ has been computed
	\STATE Compute $\gk=\brangle{-\nabla f(\xk),\sk-\xk}$ \algComment{Compute FW gap}
	\STATE Compute $\gunif = \brangle{-\nabla f(\xk),\unif-\xk}$ \algComment{Compute ``uniform gap''}
		\IF{$\gunif<0$}
		\STATE Let $\tilde{\shrinkAmount} = \frac{\gk}{-4\gunif}$  \algComment{Compute new proposal for $\shrinkAmount$}
			\IF{$\tilde{\shrinkAmount}<\epsk{k-1}$}
				\STATE $\epsk{k} = \min\left(\tilde{\shrinkAmount},\frac{\epsk{k-1}}{2}\right)$ \algComment{Shrink by at least a factor of two if proposal is smaller}
			\ENDIF
		\ENDIF	\algComment{and set $\epsk{k} = \epsk{k-1}$ if it was not updated}
	\end{algorithmic}
\end{algorithm}
We can show that the algorithm converges to the global solution as follows:
\begin{theorem}[Global convergence for adaptive-$\shrinkAmount$ variant over $\DOM$]
	\label{thm:convergence_adaptive_eps_main}
	For a function $f$ satisfying the properties in Prob.~\ref{prop:generic_fxn} and Prop.~\ref{prop:prop_bounded_grad}, the sub-optimality of the iterates obtained by running the FW updates over $\DOMeps$ with $\shrinkAmount$ updated according to Alg.~\ref{alg:adaptive_update}
	is bounded as:
	$$f(\xk)-f(\xopt)\leq O\left(k^{-\frac{1}{p+1}}\right).$$
\end{theorem}
A full proof with a precise rate and constants is given in App.~\ref{sec:theory_adaptive_eps}.
The sub-optimality $\hk := f(\xk)-f(\xopt)$ traverses three stages with an overall rate as above. 
The updates to $\shrinkAmount^{(k)}$ as in Alg.~\ref{alg:adaptive_update} 
enable us to (1) upper bound the duality gap over $\DOM$ as a function of the duality gap in $\DOMeps$ 
and (2) lower bound the value
of $\shrinkAmount^{(k)}$ as a function of $\hk$. Applying the standard Descent Lemma with the Lipschitz constant on the gradient of the form $L\shrinkAmount^{-p}$ (Prop.~\ref{prop:prop_bounded_grad}), and replacing $\shrinkAmount^{(k)}$ by its bound in $\hk$, we get the recurrence: $\hnext \leq \hk - C \hk^{p+2}$. Solving this gives us the desired bound.

\textbf{Application to the TRW Objective: }
$\min_{\vmu\in\MARG}-\TRW$ is akin to $\min_{\x\in\DOM}f(\x)$ and the
(strong) convexity of $-\TRW$ has been previously shown \citep{wainwright2005new,london_icml15}.
The gradient of the TRW objective is Lipschitz continuous over $\Meps$
since all marginals are strictly positive. Its growth for
Prop.~\ref{prop:prop_bounded_grad} can be bounded with $p=1$ as we
show in App.~\ref{sec:trw_bounded_lip}. This gives a rate of
convergence of $O(k^{-1/2})$ for the adaptive-$\shrinkAmount$ variant,
which interestingly is a typical rate for non-smooth convex optimization. The hidden constant is of the order~$O(\| \theta \| \cdot |V|)$. The modulus of continuity $\modContinuity$ for the TRW objective is close to linear (it is almost a Lipschitz function), and its constant is instead of the order $O(\| \theta \| + |V|)$.

\vspace{-2mm}
\section{Algorithm}
\vspace{-3mm}

Alg.~\ref{alg:algInfadaptive}
describes the pseudocode for our proposed algorithm to do marginal inference with $\TRW$. 
$\textrm{\textbf{minSpanTree}}$ finds the minimum spanning tree of a
weighted graph, and $\MIfunction(\vmu)$ computes the mutual information of edges of $G$ 
from the pseudomarginals in $\vmu$\footnote{The component $ij$ has value $H(\bmu_i)+H(\bmu_j)-H(\bmu_{ij}).$}
(to perform FW updates over $\vrho$ as in Alg. 2 in~\citet{wainwright2005new}).
It is worthwhile to note that our approach uses three levels of Frank-Wolfe: (1) for the (tightening) optimization
of~$\vrho$ over~$\TREEPOL$, (2) to perform approximate marginal inference, i.e for the optimization of $\vmu$ over $\MARG$, and (3) to perform the correction steps (lines~16 and~23).
\begin{algorithm}[t]
	\caption{Approximate marginal inference over $\MARG$ (solving~\eqref{eqn:upperBoundPartition}). Here $f$ is the negative TRW objective.}
	\label{alg:algInfadaptive}
	\begin{algorithmic}[1]
		\STATE Function \textbf{TRW-Barrier-FW}$(\vrho^{(0)}, \stopCrit, \epsk{\mathrm{init}},\unif)$:
		\STATE \textbf{Inputs:} Edge-appearance probabilities $\vrho^{(0)}$, $\epsk{\mathrm{init}}\leq \frac{1}{4}$ initial contraction of polytope, inner loop stopping criterion $\stopCrit$, fixed reference point $\unif$ in the interior of $\mathcal{M}$. Let $\epsk{-1} = \epsk{\mathrm{init}}$.
		\STATE Let $V := \{\unif\}$ (visited vertices),  $\x^{(0)} = \unif$ \quad (Initialize the algorithm at the uniform distribution) 
	\FOR[\emph{FW outer loop to optimize $\vrho$ over $\TREEPOL$}]{$i=0\dots\MAXRHOITS$}
		\FOR[\emph{FCFW inner loop to optimize $\x$ over $\MARG$}]{$k=0\dots \MAXITS$}
			\STATE Let $\tilde{\theta} = \nabla f(\x^{(k)};\vtheta,\vrho^{(i)})$ \algComment{Compute gradient}
			\STATE Let $\sk \in \displaystyle\argmin_{\vv \in \MARG} \,\, \brangle{\tilde{\theta},\vv}$ \algComment{Run MAP solver to compute FW vertex} 
			\STATE Compute $\gk=\brangle{-\tilde{\theta},\sk-\xk}$ \algComment{Inner loop FW duality gap}
			\IF{$\gk \leq \stopCrit$}
				\STATE \textbf{break} FCFW inner loop  \algComment{$\x^{(k)}$ is $\stopCrit$-optimal}
			\ENDIF
			\STATE $\epsk{k} = \epsk{k-1}$ \algComment{For Adaptive-$\shrinkAmount$: Run Alg. \ref{alg:adaptive_update} to modify $\shrinkAmount$}
			\STATE Let $\skeps = (1-\epsk{k})\sk + \epsk{k} \unif$ and $\dkeps = \skeps - \x^{(k)}$
				\algComment{$\shrinkAmount$-contracted quantities}
			\STATE $\x^{(k+1)} = \arg\min \{ f(\x^{(k)}+\gamma \, \dkeps) : \stepsize \in [0,1] \}$ \algComment{FW step with line search}
			\STATE Update correction polytope: $V := V \cup \{ \sk \}$
			\STATE $\x^{(k+1)} := \CORRECTION(\x^{(k+1)}, V, \epsk{k}, \vrho^{(i)})$ 
			\algComment{optional: correction step}
			\STATE $\x^{(k+1)},V_{\mathrm{search}} := \LOCALSEARCH(\x^{(k+1)}, \sk,\epsk{k}, \vrho^{(i)})$ \algComment{optional: fast MAP solver}
			\STATE Update correction polytope (with vertices from $\LOCALSEARCH$): $V := V \cup \{ V_{\mathrm{search}}\}$
		\ENDFOR
		\STATE $\vrho^{v} \leftarrow \textrm{\textbf{minSpanTree}}(\MIfunction(\x^{(k)}))$ \algComment {FW vertex of the spanning tree polytope}
		\STATE $\vrho^{(i+1)} \leftarrow \vrho^{(i)}+(\frac{i}{i+2})(\vrho^{v}-\vrho^{(i)})$ \algComment{Fixed step-size schedule FW update for $\vrho$ kept in $\mathrm{relint}(\TREEPOL$)}
		\STATE $\x^{(0)}\leftarrow\x^{(k)}$, $\quad \epsk{-1} \leftarrow \epsk{k-1}$ \algComment{Re-initialize for FCFW inner loop}
		\STATE If $i<\MAXRHOITS$ then $\x^{(0)} = \CORRECTION(\x^{(0)},V,\epsk{-1},\vrho^{(i+1)})$ %
	\ENDFOR
	\RETURN $\x^{(0)}$ and $\vrho^{(i)}$
	\end{algorithmic}
\end{algorithm}
We detail a few heuristics that aid practicality.

\textbf{Fast Local Search: }
Fast methods for MAP inference such as Iterated Conditional Modes \citep{besag1986statistical} offer a cheap,
low cost alternative to a more expensive combinatorial MAP solver. We warm start the ICM solver
with the last found vertex $\sk$ of the marginal polytope.
The subroutine $\LOCALSEARCH$ (Alg.~\ref{alg:algLocalSearch} in Appendix) performs
a fixed number of FW updates to the pseudomarginals using ICM as the (approximate) MAP solver.  

\textbf{Re-optimizing over the Vertices of~$\mathcal{M}$ (FCFW algorithm): }
As the iterations of FW progress, we keep track of the vertices of the marginal polytope 
found by Alg.~\ref{alg:algInfadaptive} in the set~$V$.
We make use of these vertices in the $\CORRECTION$ subroutine
(Alg.~\ref{alg:algReopt} in Appendix)
which re-optimizes the objective function over (a contraction of) the convex hull of the elements of $V$ (called the correction polytope).
$\x^{(0)}$ in Alg.~\ref{alg:algInfadaptive} is initialized to the uniform distribution
which is guaranteed to be in~$\MARG$ (and~$\Meps$). After updating~$\vrho$, 
we set $\x^{(0)}$ to the approximate minimizer in 
the correction polytope.
The intuition is that changing $\vrho$ by a small amount
may not substantially modify the optimal $\x^*$ (for the new $\vrho$) and that the new optimum might 
be in the convex hull of the
vertices found thus far. 
If so, $\CORRECTION$ will be able to find it without
resorting to any additional MAP calls. 
This encourages the MAP solver to search 
for new, unique vertices instead of rediscovering old ones. 

\textbf{Approximate MAP Solvers: }
We can swap out the exact MAP solver with an approximate MAP solver. 
The primal objective plus the (approximate) duality gap may no longer be an upper bound 
on the log-partition function (black-box MAP solvers could be
considered to optimize over an inner bound to the marginal polytope).
Furthermore, the gap over $\DOM$ may be negative
if the approximate MAP solver fails to find a direction of descent. Since adaptive-$\shrinkAmount$ requires 
that the gap be positive in Alg.~\ref{alg:adaptive_update},
we take the max over the last gap obtained over the correction
polytope (which is always non-negative) and the computed gap over $\DOM$ as a heuristic. 

Theoretically, one could get similar convergence rates 
as in Thm.~\ref{thm:convergence_fixed_eps_main}
and~\ref{thm:convergence_adaptive_eps_main} using an approximate
MAP solver that has a multiplicative guarantee on the gap (line~8 of Alg.~\ref{alg:algInfadaptive}),
as was done previously for FW-like algorithms (see, e.g., Thm.~C.1 in~\citet{lacoste2012block}).
With an $\epsilon$-additive error guarantee on the MAP solution, one can prove 
similar rates up to a suboptimality error of $\epsilon$.
Even if the approximate MAP solver does not provide an approximation
guarantee, if it returns an  {\em upper bound} on the value of the MAP
assignment (as do branch-and-cut solvers for integer linear programs,
or \cite{SontagEtAl_uai08}),
one can use this to obtain an upper bound on $\log Z$ (see App.~\ref{sec:approx_map_logz}).

\vspace{-2mm}
\section{Experimental Results\label{sec:expts}}
\vspace{-3mm}

\textbf{Setup: }
The L1 error in marginals is computed as:
$\zeta_{\mu} := \frac{1}{N} \sum_{i=1}^N |\mu_i(1)-\mu_i^*(1)|$. When using exact MAP inference, 
the error in $\log Z$ (denoted $\zeta_{\log Z}$) is computed
by adding the duality gap to the primal (since this guarantees us an upper bound). For approximate
MAP inference, we plot the primal objective. 
We use a non-uniform initialization of $\vrho$ computed with the Matrix Tree Theorem \citep{sontag2007new,koo2007structured}. 
We perform 10 updates to $\vrho$, optimize $\vmu$ to a duality gap of $0.5$ on $\MARG$, and always
perform correction steps. We use $\LOCALSEARCH$ only for the real-world instances.
We use the implementation of TRBP and the Junction Tree Algorithm (to compute exact marginals) in libDAI~\citep{Mooij_libDAI_10}.
Unless specified, we compute 
marginals by optimizing the TRW objective using the adaptive-$\shrinkAmount$ variant of the
algorithm (denoted in the figures as $M_{\shrinkAmount})$.

\textbf{MAP Solvers:} 
For approximate MAP, we run three solvers in parallel:
QPBO \citep{kolmogorov2007minimizing,boykov2004experimental},
TRW-S \citep{kolmogorov2006convergent} and ICM \citep{besag1986statistical} using
OpenGM \citep{andres2012opengm} and use the result that realizes the highest energy. 
For exact inference, we use \citet{gurobi} or toulbar2 \citep{allouche2010toulbar2}.

\textbf{Test Cases:} All of our test cases are on binary pairwise MRFs.
	(1) \textit{Synthetic 10 nodes cliques}: Same setup as
        \citet[Fig.~2]{sontag2007new}, with $9$ sets of $100$ instances each with coupling strength drawn from $\mathcal{U}[-\theta,\theta]$ for $\theta\in \{0.5,1,2,\ldots,8\}$.  
	(2) \textit{Synthetic Grids}: $15$ trials with $5\times5$ grids.
		We sample $\theta_i \sim\mathcal{U}[-1,1]$
          and $\theta_{ij}\in[-4,4]$ for nodes and edges. The
          potentials were $(-\theta_i,\theta_i)$ for nodes and
          $(\theta_{ij},-\theta_{ij}; -\theta_{ij},\theta_{ij})$ for edges.
	(3) \textit{Restricted Boltzmann Machines (RBMs)}: From the Probabilistic Inference Challenge 2011.\footnote{\url{http://www.cs.huji.ac.il/project/PASCAL/index.php}}
	(4) \textit{Horses}: Large ($N\approx 12000$) MRFs representing images from the Weizmann Horse Data \citep{classSpec} with potentials learned by
	\cite{domke2013learning}. 
	(5) \textit{Chinese Characters}: An image completion task from
        the KAIST Hanja2 database, compiled in OpenGM
        by~\citet{andres2012opengm}. The potentials were learned using Decision Tree Fields~\citep{nowozin2011decision}.
		The MRF is not a grid due to skip edges that tie nodes at various offsets. The potentials are a combination of submodular and 
		supermodular and therefore a harder task for inference algorithms. 

\newpage
		\centerline{\textbf{On the Optimization of $\MARG$ versus $\Meps$\label{sec:M_M_eps}}}
We compare the performance of Alg.~\ref{alg:algInfadaptive} on optimizing over $\MARG$ (with and without correction), optimizing 
over $\Meps$ with fixed-$\shrinkAmount = 0.0001$ (denoted $M_{0.0001}$) and optimizing over $\Meps$ using the adaptive-$\shrinkAmount$ variant.
These plots are averaged across all the trials for the \emph{first} iteration of optimizing over $\TREEPOL$.
We show error as a function of the number of MAP calls since this is the bottleneck
for large MRFs. 
Fig.~\ref{fig:M_M_eps}, \ref{fig:M_M_eps2} depict the results of this optimization aggregated across trials. 
We find that all
variants settle on the same average error. The adaptive $\shrinkAmount$ variant converges faster on average
followed by the fixed $\shrinkAmount$ variant. Despite relatively quick convergence for
$\MARG$ with no correction on the grids, we found that correction was crucial
to reducing the
number of MAP calls in subsequent steps of inference after updates to $\vrho$. %
As highlighted earlier, correction steps on $\MARG$ (in blue) worsen 
convergence, an effect brought about by iterates wandering too close to the boundary of $\MARG$.

\vspace{2mm}
\centerline{\textbf{On the Applicability of Approximate MAP Solvers}}

\textbf{Synthetic Grids:}
Fig.~\ref{fig:approxVsExact_l1} depicts the accuracy of approximate MAP 
solvers versus exact MAP solvers aggregated across trials for $5\times5$ grids.
The results using approximate MAP inference are competitive with those of exact
inference, even as the optimization is tightened over $\TREEPOL$. 
This is an encouraging and non-intuitive result
since it indicates that one can achieve high quality marginals through
the use of relatively cheaper approximate MAP oracles.%

\begin{figure}[t]
\vspace{-1mm}
\centering
\subfigure[\small{$\zeta_{\log Z}$: $5\times5$ grids \qquad \newline $\MARG$ vs $\MARG_{\shrinkAmount}$}]{
	\label{fig:M_M_eps}
	\includegraphics[height=4cm,width=6cm,keepaspectratio]{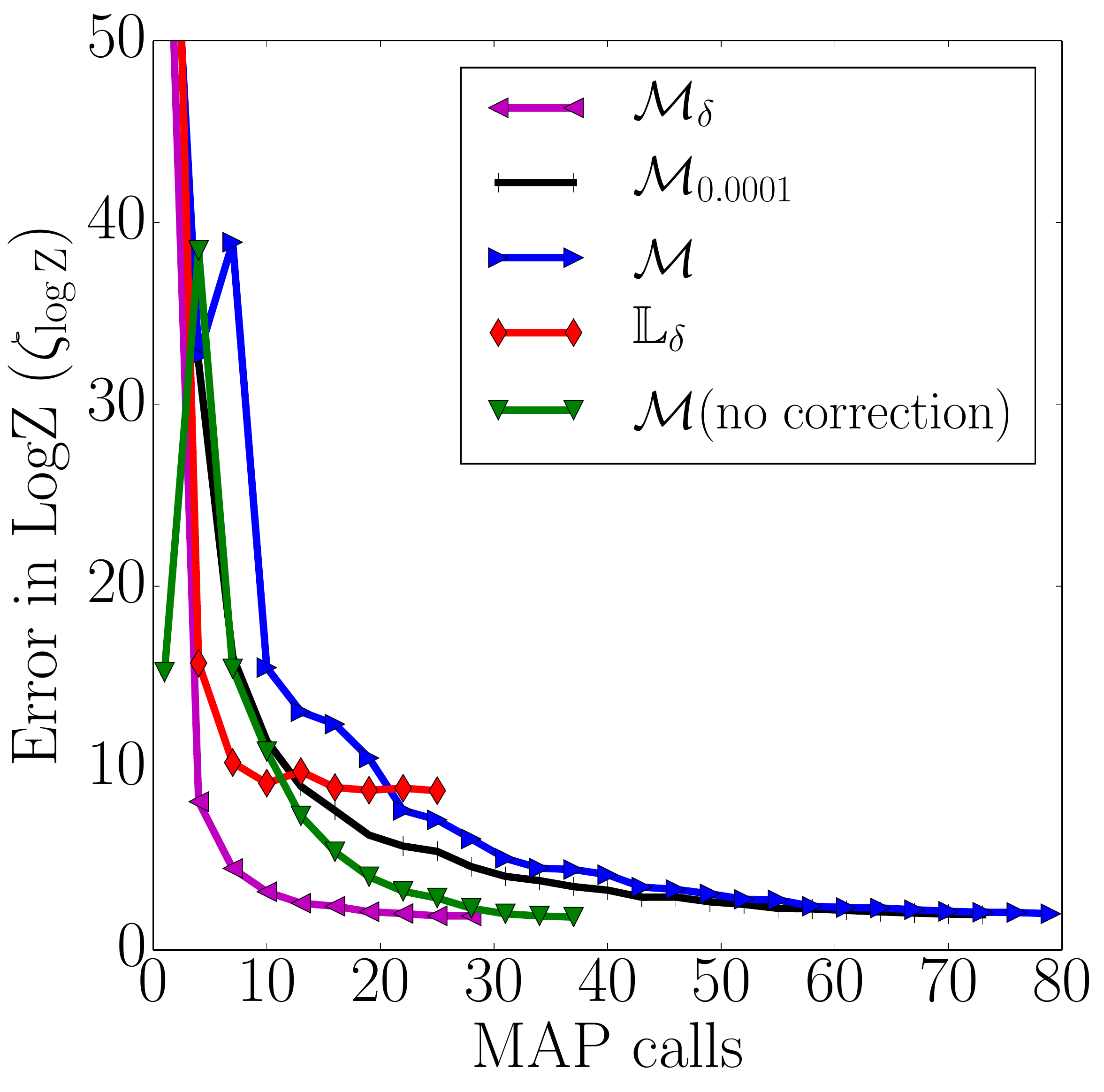}
}
\centering
\subfigure[\small{$\zeta_{\log Z}$: $10$ node cliques \qquad \newline $\MARG$ vs $\MARG_{\shrinkAmount}$}]{
	\label{fig:M_M_eps2}
\includegraphics[height=4cm,width=6cm,keepaspectratio]{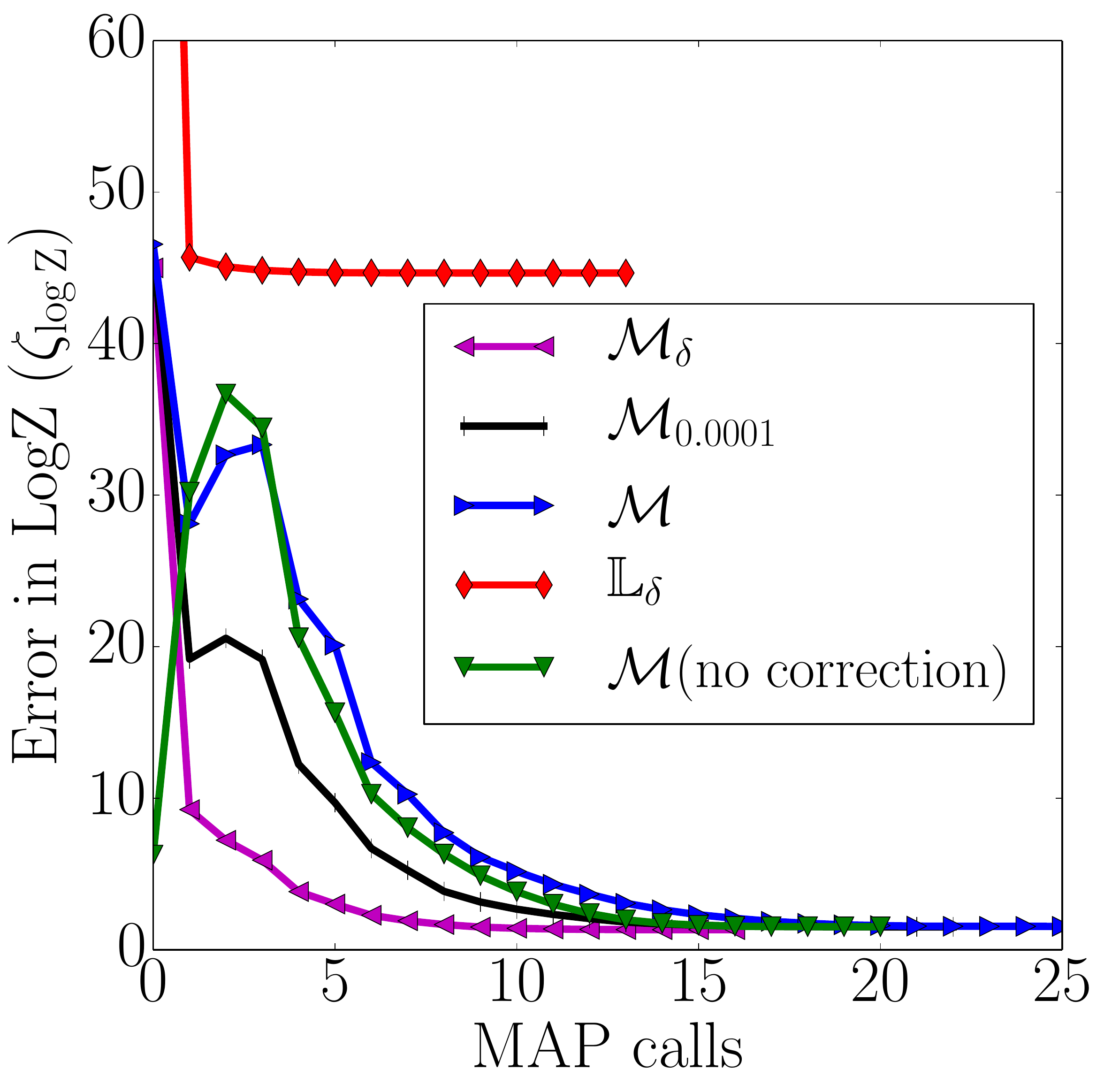}
}
\centering
\subfigure[\small{$\zeta_{\mu}$: $5\times5$ grids \newline Approx. vs. Exact MAP}]{
\label{fig:approxVsExact_l1}
	\includegraphics[height=4cm,width=6cm,keepaspectratio]{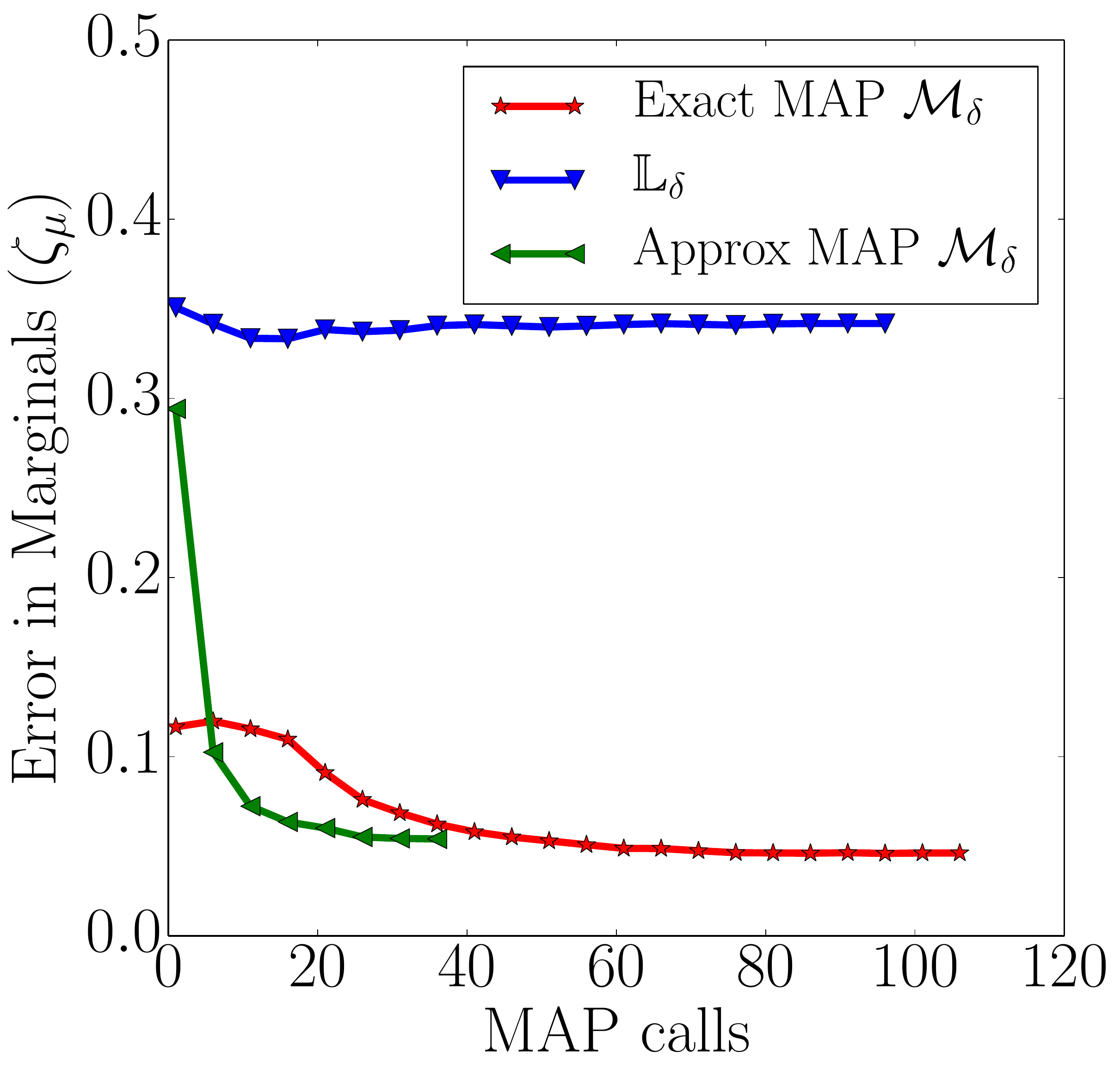}
}
\centering
\subfigure[\small{$\zeta_{\log Z}$: 40 node RBM \newline Approx. vs. Exact MAP}]{
\label{fig:rbm20_logz}
\includegraphics[height=4cm,width=6cm,keepaspectratio]{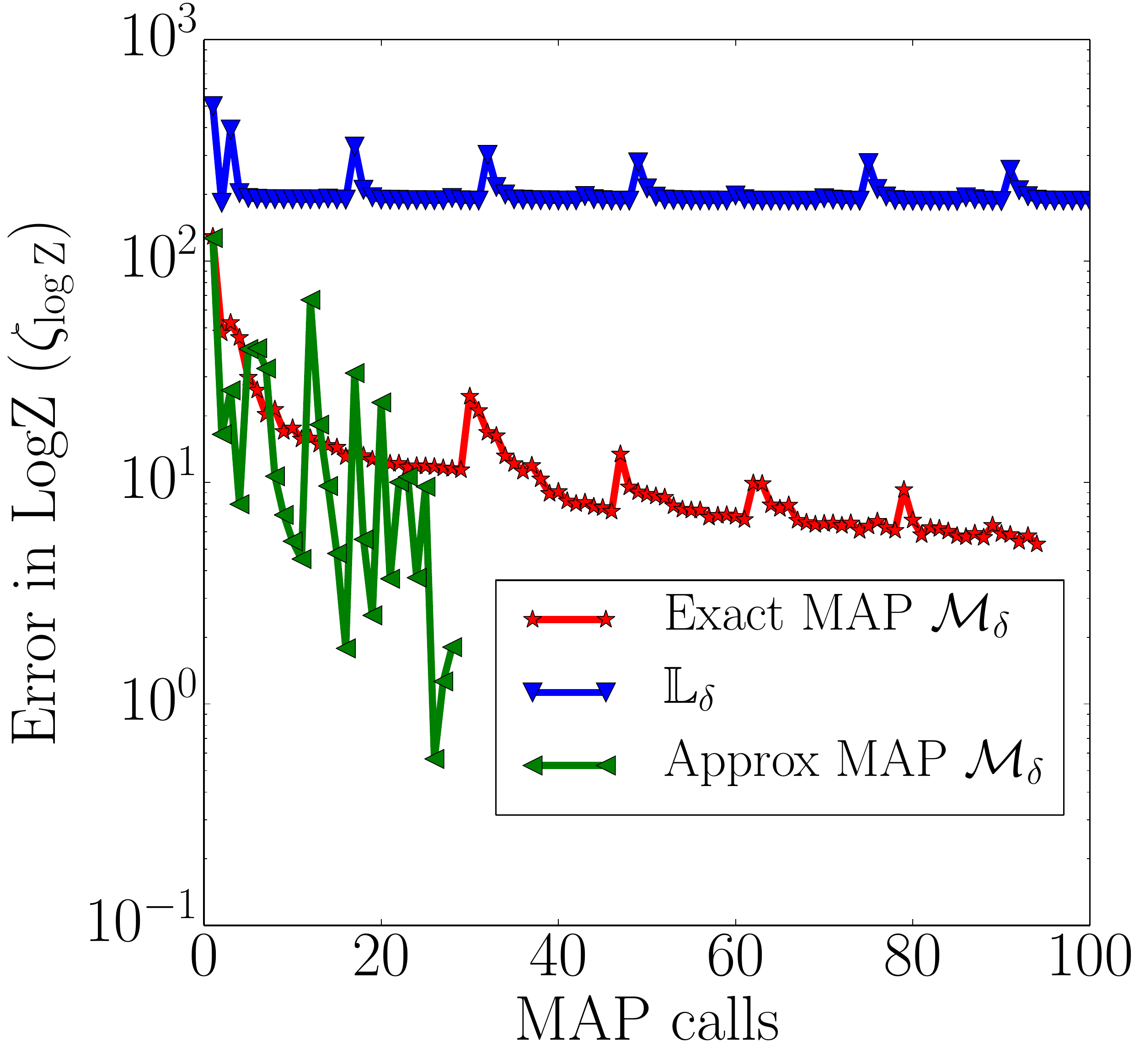}
}
\centering
\subfigure[\small{$\zeta_{\mu}$: 10 node cliques \newline Optimization over $\TREEPOL$}]{
\label{fig:syntheticComplete_tightening_l1}
\includegraphics[height=4cm,width=6cm,keepaspectratio]{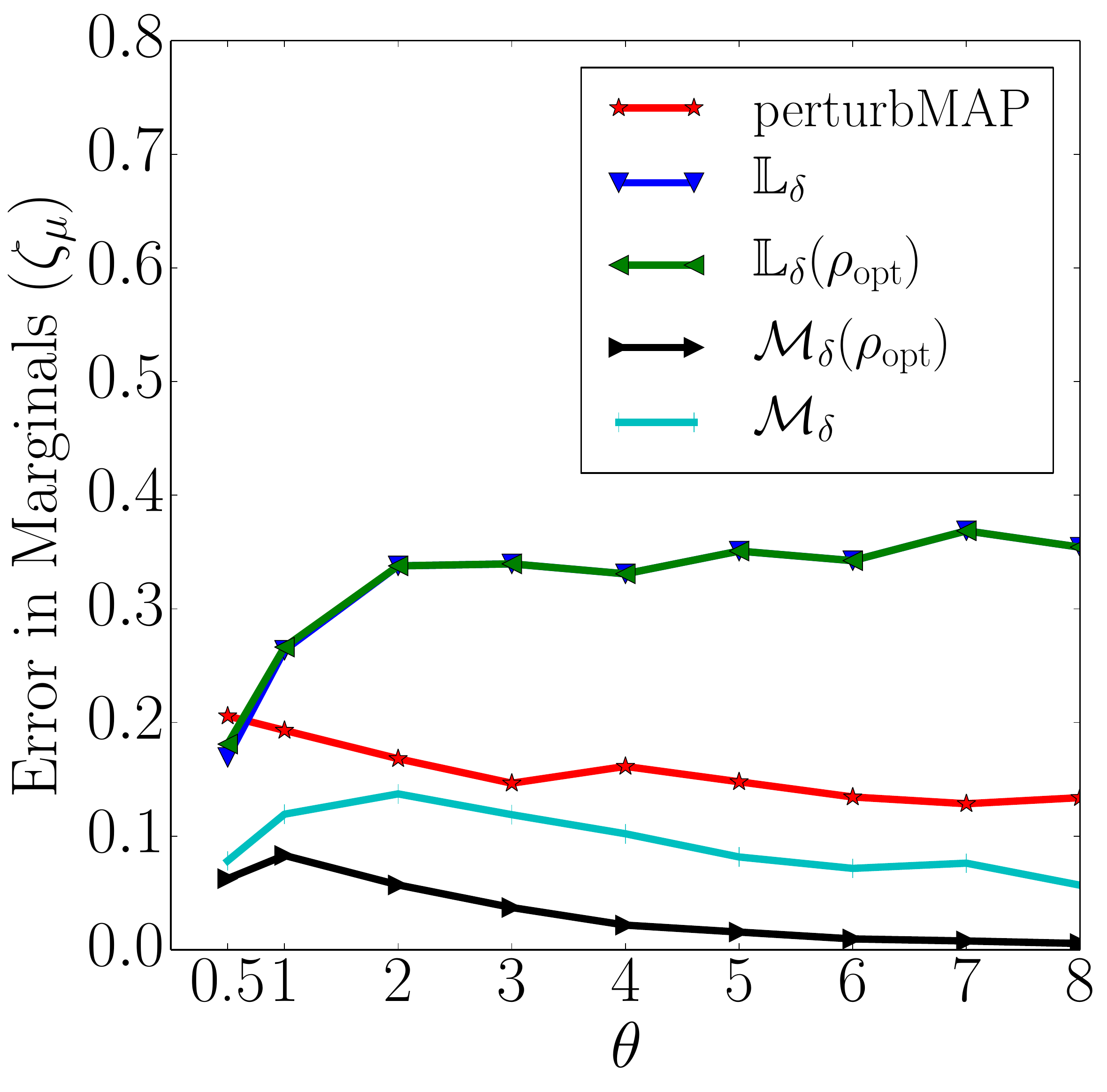}
}
\centering
\subfigure[\small{$\zeta_{\log Z}$: 10 node cliques \newline Optimization over $\TREEPOL$}]{
\label{fig:syntheticComplete_tightening_logz}
\includegraphics[height=4cm,width=6cm,keepaspectratio]{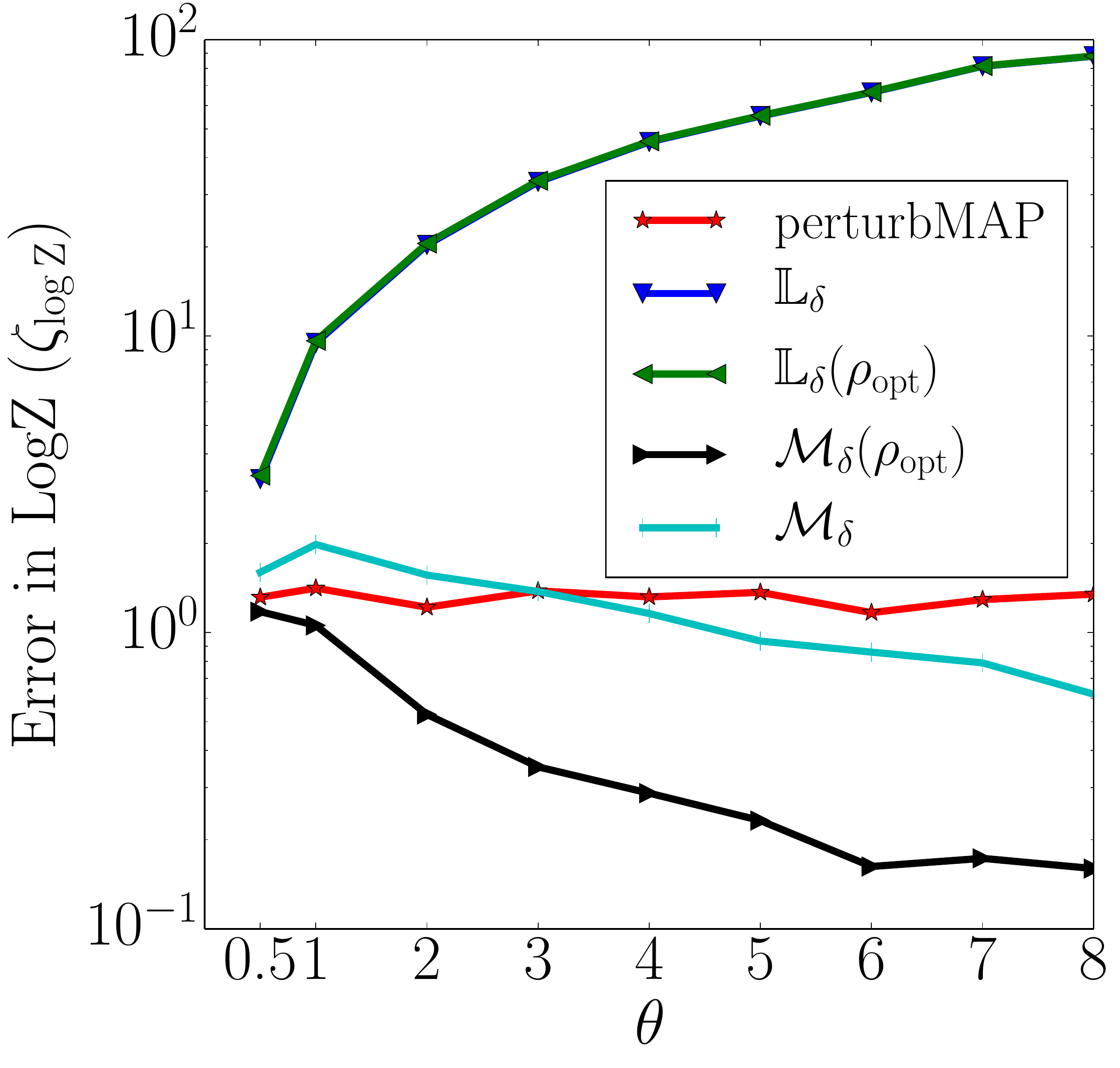}
}
\vspace{-3mm}
\caption{\small Synthetic Experiments: 
	In Fig. \ref{fig:approxVsExact_l1} \& \ref{fig:rbm20_logz}, we unravel MAP calls across updates to $\vrho$.
  Fig. \ref{fig:rbm20_logz} corresponds to a single RBM (not an aggregate over trials) where for
  ``Approx MAP'' we plot the absolute error between the primal
  objective and $\log Z$ (not guaranteed to be an upper bound).}
\vspace{-4mm}
\end{figure}

\textbf{RBMs:}
As in \cite{salakhutdinov2008learning}, we observe for RBMs that the bound provided by $\TRW$ over $\LOCAL_{\shrinkAmount}$ 
is loose and does not get better when optimizing over 
$\TREEPOL$.
As Fig.~\ref{fig:rbm20_logz} 
depicts for a single RBM,
optimizing over $\MARG_\shrinkAmount$
realizes significant gains in the upper bound on $\log Z$ which improves with updates to $\vrho$. 
The gains are preserved with the use of the approximate MAP solvers.
Note that there are also fast approximate MAP solvers specifically for RBMs \citep{sidaw2013rbm}.

\textbf{Horses:} See Fig.~\ref{fig:largeScale} (right). 
The models are close to submodular and the local relaxation is a
good approximation to the marginal polytope.
Our marginals are visually similar to those obtained by TRBP and our algorithm is able
to scale to large instances 
by using approximate MAP solvers. 
\vspace{1mm}
\centerline{\textbf{On the Importance of Optimizing over $\TREEPOL$}}

\textbf{Synthetic Cliques: }
In Fig.~\ref{fig:syntheticComplete_tightening_l1}, \ref{fig:syntheticComplete_tightening_logz},
we study the effect 
of tightening over $\TREEPOL$ against coupling strength $\theta$.
We consider the $\zeta_{\mu}$ and $\zeta_{\log Z}$ obtained for the final marginals before updating $\vrho$ (step~19)
and compare to the values obtained after optimizing over $\TREEPOL$ (marked with $\vrho_{opt}$). 
The optimization over $\TREEPOL$ has little effect on TRW optimized over $\LOCAL_\shrinkAmount$. 
For optimization over $\MARG_\shrinkAmount$,  
updating $\vrho$ realizes better marginals and bound on $\log Z$
(over and above 
those
obtained in 
\cite{sontag2007new}).

\textbf{Chinese Characters:} Fig.~\ref{fig:largeScale} (left) displays marginals across iterations of
optimizing over $\TREEPOL$.
The submodular and supermodular potentials lead to frustrated models for
which $\LOCAL_\shrinkAmount$ is very loose, which results in TRBP
obtaining poor results.\footnote{We run TRBP for 1000 iterations
  using damping = 0.9; the algorithm converges with a max norm difference between consecutive iterates of 0.002.
  Tightening over $\TREEPOL$ did not significantly change the results of TRBP.}
Our method produces reasonable marginals even before the first
update to $\vrho$, and these improve with tightening over $\TREEPOL$. 

\vspace{2mm}
\centerline{\textbf{Related Work for Marginal Inference with MAP Calls}}

\cite{hazan2012partition} estimate 
$\log Z$ by averaging MAP estimates obtained on randomly perturbed
inflated graphs. 
Our implementation of the method performed well 
in approximating $\log Z$ but the marginals (estimated by fixing the value of each
random variable and estimating $\log Z$ for the resulting graph) were
less accurate than our method (Fig.~\ref{fig:syntheticComplete_tightening_l1}, \ref{fig:syntheticComplete_tightening_logz}).

\begin{figure*}
\centering
\includegraphics[width=\textwidth]{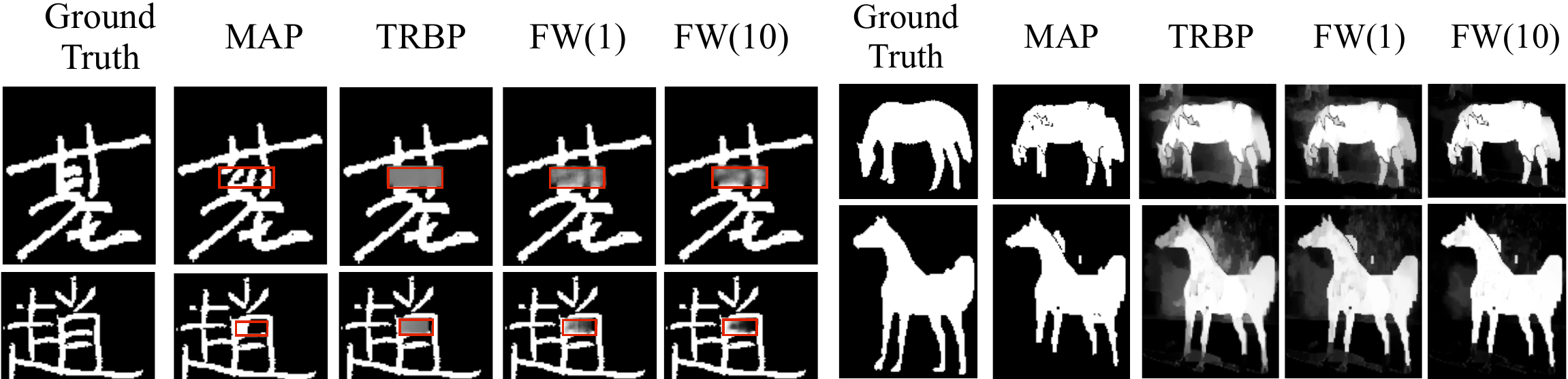}
\caption{\small{Results on real world test cases. FW(i) corresponds to the final marginals at the $i$th iteration of optimizing $\vrho$. The area highlighted on the Chinese Characters depicts the region of uncertainty.}\label{fig:largeScale}}
\end{figure*}

\vspace{-2mm}
\section{Discussion}
\vspace{-3mm}
We introduce the first provably convergent
algorithm for the TRW objective over the marginal polytope, 
under the assumption of exact MAP oracles.
We quantify the gains obtained both from marginal inference over $\MARG$ \textit{and} from tightening
over the spanning tree polytope. 
We give heuristics that improve the scalability of Frank-Wolfe
when used for marginal inference.
The runtime cost of iterative MAP calls (a reasonable rule of thumb is to assume 
an approximate MAP call takes roughly the same time as a run of TRBP) is worthwhile particularly in 
cases such as the Chinese Characters where $\LOCAL$ is loose. 
Specifically, our algorithm is appropriate for domains where marginal inference is hard but there exist
efficient MAP solvers capable of handling non-submodular potentials.
Code is available at {\small \url{https://github.com/clinicalml/fw-inference}}. 

Our work creates a flexible, modular framework for optimizing a broad
class of variational objectives, not simply TRW, with guarantees of convergence. We hope that this
will encourage more research on building better entropy approximations.
The framework we adopt is more generally applicable
to optimizing functions whose gradients tend to infinity at the boundary of the
domain.

Our method to deal with gradients that diverge at the boundary
bears resemblance to barrier functions used in interior point methods insofar as they bound
the solution away from the constraints. Iteratively decreasing $\shrinkAmount$ in our framework
can be compared to decreasing the strength of the barrier, enabling the iterates to get closer to the
facets of the polytope, although its worthwhile to note that we have an \emph{adaptive} method of doing so.

\vspace{-2mm}
\section*{Acknowledgements}
\vspace{-2mm}
RK and DS gratefully acknowledge the support of the Defense Advanced
Research Projects Agency (DARPA) Probabilistic Programming for
Advancing Machine Learning (PPAML) Program under Air Force Research
Laboratory (AFRL) prime contract no. FA8750-14-C-0005. Any opinions,
findings, and conclusions or recommendations expressed in this
material are those of the author(s) and do not necessarily reflect the
view of DARPA, AFRL, or the US government.
\bibliographystyle{abbrvnat}
\small{
\bibliography{ref}

\begin{thebibliography}{30}
\providecommand{\natexlab}[1]{#1}
\providecommand{\url}[1]{\texttt{#1}}
\expandafter\ifx\csname urlstyle\endcsname\relax
  \providecommand{\doi}[1]{doi: #1}\else
  \providecommand{\doi}{doi: \begingroup \urlstyle{rm}\Url}\fi

\bibitem[Allouche et~al.(2010)Allouche, de~Givry, and
  Schiex]{allouche2010toulbar2}
D.~Allouche, S.~de~Givry, and T.~Schiex.
\newblock Toulbar2, an open source exact cost function network solver, 2010.

\bibitem[Andres et~al.(2012)Andres, T., and Kappes]{andres2012opengm}
B.~Andres, B.~T., and J.~H. Kappes.
\newblock Opengm: A c++ library for discrete graphical models, June 2012.

\bibitem[Belanger et~al.(2013)Belanger, Sheldon, and
  McCallum]{belangerWorkshop2013}
D.~Belanger, D.~Sheldon, and A.~McCallum.
\newblock Marginal inference in {MRFs} using {F}rank-{W}olfe.
\newblock \emph{NIPS Workshop on Greedy Optimization, Frank-Wolfe and Friends},
  2013.

\bibitem[Besag(1986)]{besag1986statistical}
J.~Besag.
\newblock On the statistical analysis of dirty pictures.
\newblock \emph{J R Stat Soc Series B}, 1986.

\bibitem[Borenstein and Ullman(2002)]{classSpec}
E.~Borenstein and S.~Ullman.
\newblock Class-specific, top-down segmentation.
\newblock In \emph{ECCV}, 2002.

\bibitem[Boykov and Kolmogorov(2004)]{boykov2004experimental}
Y.~Boykov and V.~Kolmogorov.
\newblock An experimental comparison of min-cut/max-flow algorithms for energy
  minimization in vision.
\newblock \emph{TPAMI}, 2004.

\bibitem[Domke(2013)]{domke2013learning}
J.~Domke.
\newblock Learning graphical model parameters with approximate marginal
  inference.
\newblock \emph{TPAMI}, 2013.

\bibitem[Ermon et~al.(2013)Ermon, Gomes, Sabharwal, and Selman]{ermonWISH}
S.~Ermon, C.~P. Gomes, A.~Sabharwal, and B.~Selman.
\newblock Taming the {C}urse of {D}imensionality: {D}iscrete {I}ntegration by
  {H}ashing and {O}ptimization.
\newblock In \emph{ICML}, 2013.

\bibitem[Garber and Hazan(2013)]{garber2013linearly}
D.~Garber and E.~Hazan.
\newblock A linearly convergent conditional gradient algorithm with
  applications to online and stochastic optimization.
\newblock \emph{arXiv preprint arXiv:1301.4666}, 2013.

\bibitem[Globerson and Jaakkola(2007)]{globerson2012convergent}
A.~Globerson and T.~Jaakkola.
\newblock {C}onvergent {P}ropagation {A}lgorithms via {O}riented {T}rees.
\newblock In \emph{UAI}, 2007.

\bibitem[Gurobi~Optimization(2015)]{gurobi}
I.~Gurobi~Optimization.
\newblock Gurobi optimizer reference manual, 2015.

\bibitem[Hazan and Jaakkola(2012)]{hazan2012partition}
T.~Hazan and T.~Jaakkola.
\newblock On the {P}artition {F}unction and {R}andom {M}aximum {A}-{P}osteriori
  {P}erturbations.
\newblock In \emph{ICML}, 2012.

\bibitem[Jaggi(2013)]{jaggi2013revisiting}
M.~Jaggi.
\newblock Revisiting {F}rank-{W}olfe: {P}rojection-{F}ree {S}parse {C}onvex
  {O}ptimization.
\newblock In \emph{ICML}, 2013.

\bibitem[Jancsary and Matz(2011)]{Jancsary11}
J.~Jancsary and G.~Matz.
\newblock {C}onvergent {D}ecomposition {S}olvers for {T}ree-reweighted {F}ree
  {E}nergies.
\newblock In \emph{AISTATS}, 2011.

\bibitem[Kappes et~al.(2013)]{kappes2013comparative}
J.~Kappes et~al.
\newblock A comparative study of modern inference techniques for discrete
  energy minimization problems.
\newblock In \emph{CVPR}, 2013.

\bibitem[Kolmogorov(2006)]{kolmogorov2006convergent}
V.~Kolmogorov.
\newblock Convergent tree-reweighted message passing for energy minimization.
\newblock \emph{TPAMI}, 2006.

\bibitem[Kolmogorov and Rother(2007)]{kolmogorov2007minimizing}
V.~Kolmogorov and C.~Rother.
\newblock Minimizing nonsubmodular functions with graph cuts-{A} {R}eview.
\newblock \emph{TPAMI}, 2007.

\bibitem[Koo et~al.(2007)Koo, Globerson, Carreras, and
  Collins]{koo2007structured}
T.~Koo, A.~Globerson, X.~Carreras, and M.~Collins.
\newblock Structured prediction models via the matrix-tree theorem.
\newblock In \emph{EMNLP-CoNLL}, 2007.

\bibitem[Lacoste-Julien and Jaggi(2015)]{lacoste2015MFW}
S.~Lacoste-Julien and M.~Jaggi.
\newblock On the global linear convergence of {F}rank-{W}olfe optimization
  variants.
\newblock In \emph{NIPS}, 2015.

\bibitem[Lacoste-Julien et~al.(2013)Lacoste-Julien, Jaggi, Schmidt, and
  Pletscher]{lacoste2012block}
S.~Lacoste-Julien, M.~Jaggi, M.~Schmidt, and P.~Pletscher.
\newblock Block-coordinate {F}rank-{W}olfe optimization for structural {SVM}s.
\newblock In \emph{ICML}, 2013.

\bibitem[London et~al.(2015)London, Huang, and Getoor]{london_icml15}
B.~London, B.~Huang, and L.~Getoor.
\newblock The benefits of learning with strongly convex approximate inference.
\newblock In \emph{ICML}, 2015.

\bibitem[Mooij(2010)]{Mooij_libDAI_10}
J.~M. Mooij.
\newblock lib{DAI}: A free and open source {C++} library for discrete
  approximate inference in graphical models.
\newblock \emph{JMLR}, 2010.

\bibitem[Nowozin et~al.(2011)Nowozin, Rother, Bagon, Sharp, Yao, and
  Kohli]{nowozin2011decision}
S.~Nowozin, C.~Rother, S.~Bagon, T.~Sharp, B.~Yao, and P.~Kohli.
\newblock Decision tree fields.
\newblock In \emph{ICCV}, 2011.

\bibitem[Papandreou and Yuille(2011)]{papandreou2011perturb}
G.~Papandreou and A.~Yuille.
\newblock Perturb-and-map random fields: Using discrete optimization to learn
  and sample from energy models.
\newblock In \emph{ICCV}, 2011.

\bibitem[Salakhutdinov(2008)]{salakhutdinov2008learning}
R.~Salakhutdinov.
\newblock Learning and evaluating boltzmann machines.
\newblock Technical report, 2008.

\bibitem[Shimony(1994)]{MAP_NP_Hard}
S.~Shimony.
\newblock Finding {MAP}s for {B}elief {N}etworks is {NP}-hard.
\newblock \emph{Artificial Intelligence}, 1994.

\bibitem[Sontag and Jaakkola(2007)]{sontag2007new}
D.~Sontag and T.~Jaakkola.
\newblock New outer bounds on the marginal polytope.
\newblock In \emph{NIPS}, 2007.

\bibitem[Sontag et~al.(2008)Sontag, Meltzer, Globerson, Weiss, and
  Jaakkola]{SontagEtAl_uai08}
D.~Sontag, T.~Meltzer, A.~Globerson, Y.~Weiss, and T.~Jaakkola.
\newblock Tightening {LP} relaxations for {MAP} using message-passing.
\newblock In \emph{UAI}, 2008.

\bibitem[Wainwright et~al.(2005)Wainwright, Jaakkola, and
  Willsky]{wainwright2005new}
M.~J. Wainwright, T.~S. Jaakkola, and A.~S. Willsky.
\newblock A new class of upper bounds on the log partition function.
\newblock \emph{IEEE Transactions on Information Theory}, 2005.

\bibitem[Wang et~al.(2014)Wang, Frostig, Liang, and Manning]{sidaw2013rbm}
S.~Wang, R.~Frostig, P.~Liang, and C.~Manning.
\newblock Relaxations for inference in restricted {B}oltzmann machines.
\newblock In \emph{ICLR Workshop}, 2014.

\end{thebibliography}


\begin{thebibliography}{5}
\providecommand{\natexlab}[1]{#1}
\providecommand{\url}[1]{\texttt{#1}}
\expandafter\ifx\csname urlstyle\endcsname\relax
  \providecommand{\doi}[1]{doi: #1}\else
  \providecommand{\doi}{doi: \begingroup \urlstyle{rm}\Url}\fi

\bibitem[Bertsekas(1999)]{bertsekas1999nonlinear}
D.~P. Bertsekas.
\newblock \emph{Nonlinear programming}.
\newblock Athena Scientific, Belmont, MA, 1999.

\bibitem[Gu{\'e}lat and Marcotte(1986)]{Guelat:1986fq}
J.~Gu{\'e}lat and P.~Marcotte.
\newblock Some comments on {W}olfe's {\textquoteleft}away
  step{\textquoteright}.
\newblock \emph{Mathematical Programming}, 35\penalty0 (1):\penalty0 110--119,
  1986.

\bibitem[Nesterov(2004)]{nesterov2004lectures}
Y.~Nesterov.
\newblock \emph{Introductory Lectures on Convex Optimization}.
\newblock Kluwer Academic Publishers, Norwell, MA, 2004.

\bibitem[Teo et~al.(2007)Teo, Smola, Vishwanathan, and Le]{teo2007erm}
C.~Teo, A.~Smola, S.~Vishwanathan, and Q.~Le.
\newblock A scalable modular convex solver for regularized risk minimization.
\newblock In \emph{KDD}, 2007.

\bibitem[Wolfe(1970)]{Wolfe:1970wy}
P.~Wolfe.
\newblock {Convergence Theory in Nonlinear Programming}.
\newblock In J.~Abadie, editor, \emph{Integer and Nonlinear Programming}, pages
  1--23. North-Holland, 1970.

\end{thebibliography}
}

\clearpage %
\appendix
\section{Preliminaries}

\subsection{Summary of Supplementary Material}
The supplementary material is divided into two parts: 

(1) The first part is dedicated to the exposition
of the theoretical results presented in the main paper. 
Section~\ref{sec:FW_alg} details the variants of the Frank-Wolfe algorithm that we used and analyzed.
Section~\ref{sec:theory_fixed_eps} gives the proof to Theorem~\ref{thm:convergence_fixed_eps_main} (fixed $\shrinkAmount$) while Section~\ref{sec:theory_adaptive_eps} gives the proof to Theorem~\ref{thm:convergence_adaptive_eps_main} (adaptive $\shrinkAmount$). Finally, Section~\ref{sec:trw_properties} applies the convergence theorem to the TRW objective and investigates the relevant constants.

(2) The remainder of the supplementary material provides more information about the experimental setup as well as additional experimental results.

\subsection{Descent Lemma}
The following descent lemma is proved in \citesup{bertsekas1999nonlinear} (Prop. A24) and is standard for any convergence proof of first order methods. We provide a proof here for completeness. It also highlights the origin of the requirement that we use dual norm pairings between $\x$ and the gradient of $f(\x)$ (because of the generalized Cauchy-Schwartz inequality).
\begin{lemma}
	\label{lem:descent_lemma}
	\textbf{Descent Lemma}
	
	Let $\x_\stepsize :=  \x + \stepsize \dd$ and suppose that $f$ is continuously differentiable on the line segment from $\x$ to ${\x}_{\stepmax}$ for some $\stepmax > 0$.
Suppose that $L = {\sup}_{\alpha \in ]0,\stepmax]} \frac{|| \nabla f(\x+\alpha\dd)-\nabla f(\x)||_*}{||\alpha \dd||}$ is finite, then we have: 
	\begin{equation} \label{eq:descent_lemma}
	f(\x_\stepsize)\leq f(\x)+\stepsize \brangle{\nabla f(\x),\dd} + \frac{\stepsize^2}{2} L||\dd||^2, \quad \forall \stepsize \in [0, \stepmax].
	\end{equation}
\end{lemma}
\begin{proof}
	Let $0 < \stepsize \leq \stepmax$. Denoting $l(\alpha) =f(\x+\alpha\dd)$, we have that:
	\begin{dmath*}
		f(\x_\stepsize)-f(\x)=l(\stepsize)-l(0)=\int_{0}^{\stepsize}\nabla_{\alpha}l(\alpha)d\alpha\\
		=\int_{0}^{\stepsize} \brangle{\dd,\nabla f(\x+\alpha\dd)} d\alpha\\  %
		= \int_{0}^{\stepsize} \brangle{\dd,\nabla f(\x)} d\alpha+ 
		\int_{0}^{\stepsize} \dd^T (\nabla f(\x+\alpha\dd)-\nabla f(\x))d\alpha\\
		\leq \int_{0}^{\stepsize} \brangle{\dd,\nabla f(\x)}d\alpha+ 
		\left|\int_{0}^{\stepsize} \dd^T (\nabla f(\x+\alpha\dd)-\nabla f(\x))\right|d\alpha\\
		\leq \int_{0}^{\stepsize} \brangle{\dd,\nabla f(\x)}d\alpha+ 
		\int_{0}^{\stepsize} ||\dd||\;\;||\nabla f(\x+\alpha\dd)-\nabla f(\x)||_* \, d\alpha\\ %
		= \stepsize\brangle{\dd,\nabla f(\x)} +  
		\int_{0}^{\stepsize} ||\dd||\;\;\frac{||\nabla f(\x+\alpha\dd)-\nabla f(\x)||_*}{\alpha||\dd||}\alpha||\dd|| \, d\alpha\\
		\leq \stepsize\brangle{\dd,\nabla f(\x)} +  
		\int_{0}^{\stepsize} ||\dd||\;\;L||\dd||\alpha \, d\alpha\\
		= \stepsize\brangle{\dd,\nabla f(\x)} +  \frac{L}{2}\stepsize^2||\dd||^2\\
	\end{dmath*}
	Rearranging terms, we get the desired bound.
\end{proof}

\section{Frank-Wolfe Algorithms} \label{sec:FW_alg}

In this section, we present the various algorithms that we use to do fully corrective Frank-Wolfe (FCFW) with adaptive contractions over the domain $\DOM$, as was done in our experiments. 

\subsection{Overview of the Modified Frank-Wolfe Algorithm (FW with Away Steps)}
To implement the approximate correction steps in the fully corrective Frank-Wolfe (FCFW) algorithm, we use the Frank-Wolfe algorithm with away steps~\citepsup{Wolfe:1970wy}, also known as the modified Frank-Wolfe (MFW) algorithm~\citepsup{Guelat:1986fq}. 
We give pseudo-code for MFW in Algorithm~\ref{alg:MFW} (taken from~\citep{lacoste2015MFW}). 
This variant of Frank-Wolfe adds the possibility to do an ``away step'' (see step~5 in Algorithm~\ref{alg:MFW}) in order to avoid the zig zagging phenomenon that slows down Frank-Wolfe when the solution is close to the boundary of the polytope. 
For a strongly convex objective (with Lipschitz continuous gradient), the MFW was known to have asymptotic linear convergence~\citepsup{Guelat:1986fq} and its global linear convergence rate was shown recently~\citep{lacoste2015MFW}, accelerating the slow general sublinear rate of Frank-Wolfe. 
When performing a correction over the convex hull over a (somewhat small) set of vertices of $\DOMeps$, this convergence difference was quite significant in our experiments (MFW converging in a small number of iterations to do an approximate correction vs. FW taking hundreds of iterations to reach a similar level of accuracy). 
We note that the TRW objective is strongly convex when all the edge probabilities are non-zero~\citep{wainwright2005new}; and that it has Lipschitz gradient over $\DOMeps$ (but not $\DOM$).

The gap computed in step~6 of~Algorithm~\ref{alg:MFW} is non-standard; it is a sufficient condition to ensure the global linear convergence of the outer FCFW algorithm when using Algorithm~\ref{alg:MFW} as a subroutine to implement the approximate correction step. See~\citet{lacoste2015MFW} for more details.

The MFW algorithm requires more bookkeeping than standard FW: in addition to the current iterate $\xk$, it also maintains both the active set $\Coreset^{(k)}$ (to search for the ``away vertex'') as well as the barycentric coordinates $\bm{\alpha}^{(k)}$ (to know what are the away step-sizes that ensure feasibility -- see step~13) i.e. $\xk= \sum_{\vv \in \Coreset^{(k)}} \alpha^{(k)}_{\vv} \vv$.

\begin{algorithm}
	\caption{Modified Frank-Wolfe algorithm (FW with Away Steps) -- used for approximate correction}
	\label{alg:MFW}
	\begin{algorithmic}[1]
	\STATE Function \textbf{MFW}$(\x^{(0)}, \bm{\alpha}^{(0)}, \Vertices, \stopCrit)$ to optimize over $\conv(\Vertices)$: 
	\STATE \textbf{Inputs:} Set of atoms $\Vertices$, starting point $\x^{(0)}= \sum_{\vv \in \Coreset^{(0)}} \alpha^{(0)}_{\vv} \vv$ where $\Coreset^{(0)}$ is active set and $\bm{\alpha}^{(0)}$ the active coordinates, stopping criterion $\stopCrit$.
	\FOR{$k=0\dots K$}
		\STATE Let $\s_k \in \displaystyle\argmin_{\vv \in \Vertices} \textstyle\brangle{\nabla f(\x^{(k)}),\vv}$ and $\dd_k^\FW := \s_k - \x^{(k)}$ \qquad~~ \emph{\small(the FW direction)}
		\STATE Let $\vv_k \in \displaystyle\argmax_{\vv \in \Coreset^{(k)} } \textstyle\left\langle \nabla f(\x^{(k)}), \vv \right\rangle$ and $\dd_k^\away := \x^{(k)} - \vv_k$ \qquad \emph{\small(the away direction)}
		\STATE $g_k^\pFW  := \left\langle -\nabla f(\x^{(k)}), \dd_k^\FW + \dd_k^\away\right\rangle$ \qquad \emph{\small(stringent gap is FW + away gap to work better for FCFW)}
		\IF{$g_k^\pFW \leq \stopCrit$}
			\STATE \textbf{return} $\x^{(k)}$, $\bm{\alpha}^{(k)}$, $\Coreset^{(k)}$.
		\ELSE
  				\IF{$\left\langle -\nabla f(\x^{(k)}), \dd_k^\FW\right\rangle  \geq \left\langle -\nabla f(\x^{(k)}), \dd_k^\away\right\rangle$ }
		 		  \STATE $\dd_k :=  \dd_k^\FW$, and $\stepmax := 1$  
		 			     \hspace{20mm}\emph{\small(choose the FW direction)}
		 		  \ELSE
		 		  \STATE $\dd_k :=  \dd_k^\away$, and $\stepmax := \frac{\alpha_{\vv_k}}{(1- \alpha_{\vv_k})}$
		 		  	\hspace{12mm}\emph{\small(choose away direction; maximum feasible step-size)}
		 		  \ENDIF	
		 		  \STATE Line-search: $\stepsize_k \in \displaystyle\argmin_{\stepsize \in [0,\stepmax]} \textstyle f\left(\x^{(k)} + \stepsize \dd_k\right)$
			\STATE Update $\x^{(k+1)} := \x^{(k)} + \stepsize_k \dd_k$
			\STATE Update coordinates $\bm{\alpha}^{(k+1)}$ accordingly (see \citet{lacoste2015MFW}).
			\STATE Update $\Coreset^{(k+1)} := \{\vv \: s.t. \: \alpha^{(k+1)}_{\vv} > 0\}$
		 \ENDIF	
	\ENDFOR

	\end{algorithmic}
\end{algorithm}

\subsection{Fully Corrective Frank-Wolfe (FCFW) with Adaptive-$\shrinkAmount$}

We give in Algorithm~\ref{alg:adaptive_eps} the pseudo-code to perform fully corrective Frank-Wolfe optimization over $\DOM$ by iteratively optimizing over $\DOMeps$ with adaptive-$\shrinkAmount$ updates.
If $\shrinkAmount$ is kept constant (skipping step 10), then Algorithm~\ref{alg:adaptive_eps} implements the
fixed~$\shrinkAmount$ variant over $\DOMeps$. We describe the algorithm as maintaining the correction set of atoms~$V^{(k+1)}$ over $\DOM$ (rather than $\DOMeps$), as $\shrinkAmount$ is constantly changing. One can easily move back and forth between $V^{(k+1)}$ and its contraction $V_\shrinkAmount = (1-\epsk{k})V^{(k+1)} + \epsk{k} \unif$, and so we note that an efficient implementation might work with either representation cheaply (for example, by storing only $V^{(k+1)}$ and $\shrinkAmount$, not the perturbed version of the correction polytope). The approximate correction over $V_\shrinkAmount$ is implemented using the MFW algorithm described in Algorithm~\ref{alg:MFW}, which requires a barycentric representation $\bm{\alpha}^{(k)}$ of the current iterate $\xk$ over the correction polytope $V_\shrinkAmount$. Our notation in Algorithm~\ref{alg:adaptive_eps} uses the elements of $\Vertices$ as indices, rather than their contracted version; that is, we maintain the property that $\xk = \sum_{\vv \in \Vertices} \alpha_{\vv}^{(k)} [(1-\epsk{k}) \vv + \epsk{k} \unif]$. As $V_\shrinkAmount$ changes when $\shrinkAmount$ changes, we need to update the barycentric representation of $\xk$ accordingly -- this is done in step~11 with the following equation. Suppose that we decrease $\shrinkAmount$ to $\shrinkAmount'$. Then the old coordinates $\bm{\alpha}$ can be updated to new coordinates $\bm{\alpha}'$ for the new contraction polytope as follows:
\begin{align}
\begin{split} \label{eq:alpha_update}
\alpha'_{\vv} &= \alpha_{\vv} \frac{1-\shrinkAmount}{1-\shrinkAmount'} \qquad \text{for} \quad \vv \in \Vertices \setminus \{\unif \}, \\
\alpha'_{\unif} &= 1-\sum_{\vv \neq \unif} \alpha'_{\vv}.
\end{split}
\end{align} 
This ensures that $\sum_{\vv} \alpha_{\vv} \vv_{(\shrinkAmount)} = \sum_{\vv} \alpha'_{\vv} \vv_{(\shrinkAmount')}$, where $\vv_{(\shrinkAmount)} := (1-\shrinkAmount) \vv + \shrinkAmount \unif$, and that the coordinates form a valid convex combination (assuming that $\shrinkAmount' \leq \shrinkAmount$), as can be readily verified.

\begin{algorithm}
	\caption{Optimizing $f$ over $\DOM$ using Fully Corrective Frank-Wolfe (FCFW) with Adaptive-$\shrinkAmount$ Algorithm.}
	\label{alg:adaptive_eps}
	\begin{algorithmic}[1]
		\STATE \textbf{FCFW}$(\x^{(0)}, \Vertices, \stopCrit, \epsk{\mathrm{init}})$
		\STATE \textbf{Inputs:} Set of atoms $\Vertices$ so that $\DOM = \conv({\Vertices})$, active set $\Coreset^{(0)}$, starting point $\x^{(0)}= \sum_{\vv \in \Coreset^{(0)}} \alpha^{(0)}_{\vv}  [(1-\epsk{\mathrm{init}})\vv + \epsk{\mathrm{init}} \unif]$ where $\bm{\alpha}^{(0)}$ are the active coordinates, $\epsk{\mathrm{init}}\leq \frac{1}{4}$ describes the initial contraction of the polytope, stopping criterion $\stopCrit$, $\unif$ is a fixed reference point in the relative interior of $\DOM$.
		\STATE Let $V^{(0)} := \Coreset^{(0)}$ \quad (optionally, a bigger $V^{(0)}$ could be passed as argument for a warm start), $\epsk{-1} := \epsk{\mathrm{init}}$ 
	\FOR{$k=0\dots K$}
		\STATE Let $\sk \in \displaystyle\argmin_{\vv \in \Vertices} \textstyle\innerProd{\nabla f(\x^{(k)})}{\vv}$ \qquad \emph{\small(the FW vertex)} 
		\STATE Compute $\gk=\brangle{-\nabla f(\xk),\sk-\xk}$ \quad \emph{\small(FW gap)} 
		\IF{$\gk \leq \stopCrit$}
			\STATE \textbf{return} $\x^{(k)}$
		\ENDIF
		\STATE Let $\epsk{k}$ be $\epsk{k-1}$ updated according to Algorithm~\ref{alg:adaptive_update}.
		\STATE Update $\bm{\alpha}^{(k)}$ accordingly (using~\eqref{eq:alpha_update})
		\STATE Let $\skeps := (1-\epsk{k})\sk + \epsk{k} \unif$ 
		\STATE Let $\dd_k^\FW := \skeps - \x^{(k)}$ 
		\STATE Line-search: $\stepsize_k \in \displaystyle\argmin_{\stepsize \in [0,1]} \textstyle f\left(\x^{(k)} + \stepsize \dd_k^\FW\right)$
		\STATE Set $\x^{(\mathrm{temp)}} := \x^{(k)} + \stepsize_k  \dd_k^\FW$  \algComment{initialize correction to the update after a FW step with line search}
		\STATE $\bm{\alpha}^{(\mathrm{temp})} = (1-\stepsize_k) \bm{\alpha}^{(k)}$
		\STATE $\alpha_{\sk}^{(\mathrm{temp)}} \leftarrow \alpha_{\sk}^{(\mathrm{temp)}} + \stepsize_k$ \algComment{update coordinates according to the FW step}
		\STATE Update (non-contracted) correction polytope: $V^{(k+1)} := V^{(k)} \cup \{ \sk \}$
		\STATE Let $V_\shrinkAmount = (1-\epsk{k})V^{(k+1)} + \epsk{k} \unif$ \algComment{contracted correction polytope}
		\STATE $\x^{(k+1)}\!\!$, $\bm{\alpha}^{(k+1)} := \textbf{\textrm{MFW}}(\x^{(\mathrm{temp)}},\bm{\alpha}^{(\mathrm{temp})}, V_\shrinkAmount, \stopCrit)$ \quad \emph{\small (approximate correction step on $V_\shrinkAmount$ using MFW)}
	\ENDFOR
	\end{algorithmic}
\end{algorithm}

\section{Bounding the Sub-optimality for Fixed $\shrinkAmount$ \label{sec:theory_fixed_eps} Variant}
The pseudocode for optimizing over~$\DOMeps$ for a fixed~$\shrinkAmount$ is given in Algorithm~\ref{alg:adaptive_eps} (by ignoring the step~10 which updates $\shrinkAmount$).
It is stated with a stopping criterion $\stopCrit$, but it can alternatively
be run for a fixed number of $K$ iterations.
The following theorem bounds the suboptimality of the iterates with respect to the 
true optimum $\x^*$ over $\DOM$. If one can compute the constants in the theorem,
one can choose a target contraction amount $\shrinkAmount$ to guarantee
a specific suboptimality of $\stopCrit'$; otherwise, one can choose $\shrinkAmount$
using heuristics. Note that unlike the adaptive-$\shrinkAmount$ variant, this algorithm
does not converge to the true solution as $K \rightarrow \infty$ unless $\x^*$ happens
to belong to $\DOMeps$. But the error can be controlled by choosing $\shrinkAmount$ small enough.

\begin{theorem}[Suboptimality bound for fixed-$\shrinkAmount$ algorithm]
  \label{thm:convergence_fixed_eps}
  Let $f$ satisfy the properties in Problem~\ref{prop:generic_fxn} and suppose its gradient is Lipschitz continuous on the contractions $\DOMeps$ as in Property~\ref{prop:prop_bounded_grad}. Suppose further that $f$ is finite on the boundary of $\DOM$. 
  
  Then $f$ is uniformly continuous on $\DOM$ and has a \emph{modulus of continuity} function $\modContinuity$ quantifying its level of continuity, i.e. 
  $|f(\x) - f(\x')| \leq \modContinuity(\|\x - \x'\|) \,\, \forall \x, \x' \in \DOM$, with $\modContinuity(\sigma) \downarrow 0$ as $\sigma \downarrow 0$.
  
  Let $\x^*$ be an optimal point of $f$ over $\DOM$. The iterates $\xk \in \DOMeps$ of the FCFW algorithm as described in Algorithm~\ref{alg:adaptive_eps} for a fixed $\shrinkAmount > 0$ has sub-optimality over~$\DOM$ bounded as:
	\begin{equation} \label{eq:rate_fixed_delta}
	f(\xk) - f(\x^*) \leq \frac{2\mathcal{C}_{\shrinkAmount}}{(k+2)}+\modContinuity \left( \shrinkAmount \diam(\DOM)\right), 
	\end{equation}
	 where $C_{\shrinkAmount} \leq \diam(\DOMeps)^2 L_{\shrinkAmount}$. Note that different norms can be used in the definition of $\modContinuity(\cdot)$ and $C_\shrinkAmount$.
\end{theorem}
\begin{proof}
Let $\xopteps$ be an optimal point of $f$ over $\DOMeps$. As $f$ has a Lipschitz continuous gradient over $\DOMeps$, we can use any standard convergence result of the Frank-Wolfe algorithm to bound the suboptimality of the iterate $\xk$ over $\DOMeps$. Algorithm~\ref{alg:adaptive_eps} (with a fixed~$\shrinkAmount$) describes the FCFW algorithm which guarantees at least as much progress as the standard FW algorithm (by step~15 and~20a), and thus we can use the convergence result from~\citet{jaggi2013revisiting} as
already stated in~\eqref{eqn:fw_theorem_convergence_original}: $f(\xk)-f(\xopteps) \leq \frac{2C_{\shrinkAmount}}{(k+2)}$ with $C_{\shrinkAmount} \leq \diam(\DOMeps)^2 L_{\shrinkAmount}$, where $L_{\shrinkAmount}$ comes from Property~\ref{prop:prop_bounded_grad}. This gives the first term in~\eqref{eq:rate_fixed_delta}. Note that if the function $f$ is \emph{strongly} convex, then the FCFW algorithm has also a linear convergence rate~\citep{lacoste2015MFW}, though we do not cover this here.

We now need to bound the difference $f(\xopteps) - f(\xopt)$ coming from the fact that we are not optimizing
over the full domain, and giving the second term in~\eqref{eq:rate_fixed_delta}. We let $\xoptshift$ be the contraction of $\xopt$ on $\DOMeps$ towards $\unif$, i.e.
$\xoptshift := (1-\shrinkAmount) \xopt + \shrinkAmount \unif$.\footnote{Note that without a strong convexity assumption on $f$, the optimum over $\DOMeps$, $\xopteps$, could be quite far from the optimum over $\DOM$, $\xopt$, which is why we need to construct this alternative close point to $\xopt$.} Note that $\| \xoptshift - \xopt \| = \delta \| \xopt - \unif \| \leq  \delta \diam(\DOM)$, and thus can be made arbitrarily small by letting $\shrinkAmount \downarrow 0$. Because $\xoptshift \in \DOMeps$, we have that $f(\xoptshift) \geq f(\xopteps)$ as $\xopteps$ is optimal over~$\DOMeps$.
Thus $f(\xopteps) - f(\xopt) \leq f(\xoptshift) - f(\xopt) \leq \modContinuity(\| \xoptshift - \xopt\|)$ by the uniform continuity of~$f$ (that we explain below). Since~$\modContinuity$ is an increasing function, we have $\modContinuity(\| \xoptshift - \xopt\|) \leq \modContinuity(\delta \diam(\DOM))$, giving us the control on the second term of~\eqref{eq:rate_fixed_delta}. See Figure~\ref{fig:fixed_eps_diagram} for an illustration of the four points considered in this proof.

Finally, we explain why $f$ is uniformly continuous. As $f$ is a (lower semi-continuous) convex function, it is continuous at every point where it is finite. As $f$ is said to be finite at its boundary (and it is obviously finite in the relative interior of $\DOM$ as it is continuously differentiable there), then $f$ is continuous over the whole of $\DOM$. As $\DOM$ is compact, this means that $f$ is also uniformly continuous over $\DOM$.
\end{proof}

\begin{figure} 
\centering
\includegraphics[width=.5\textwidth]{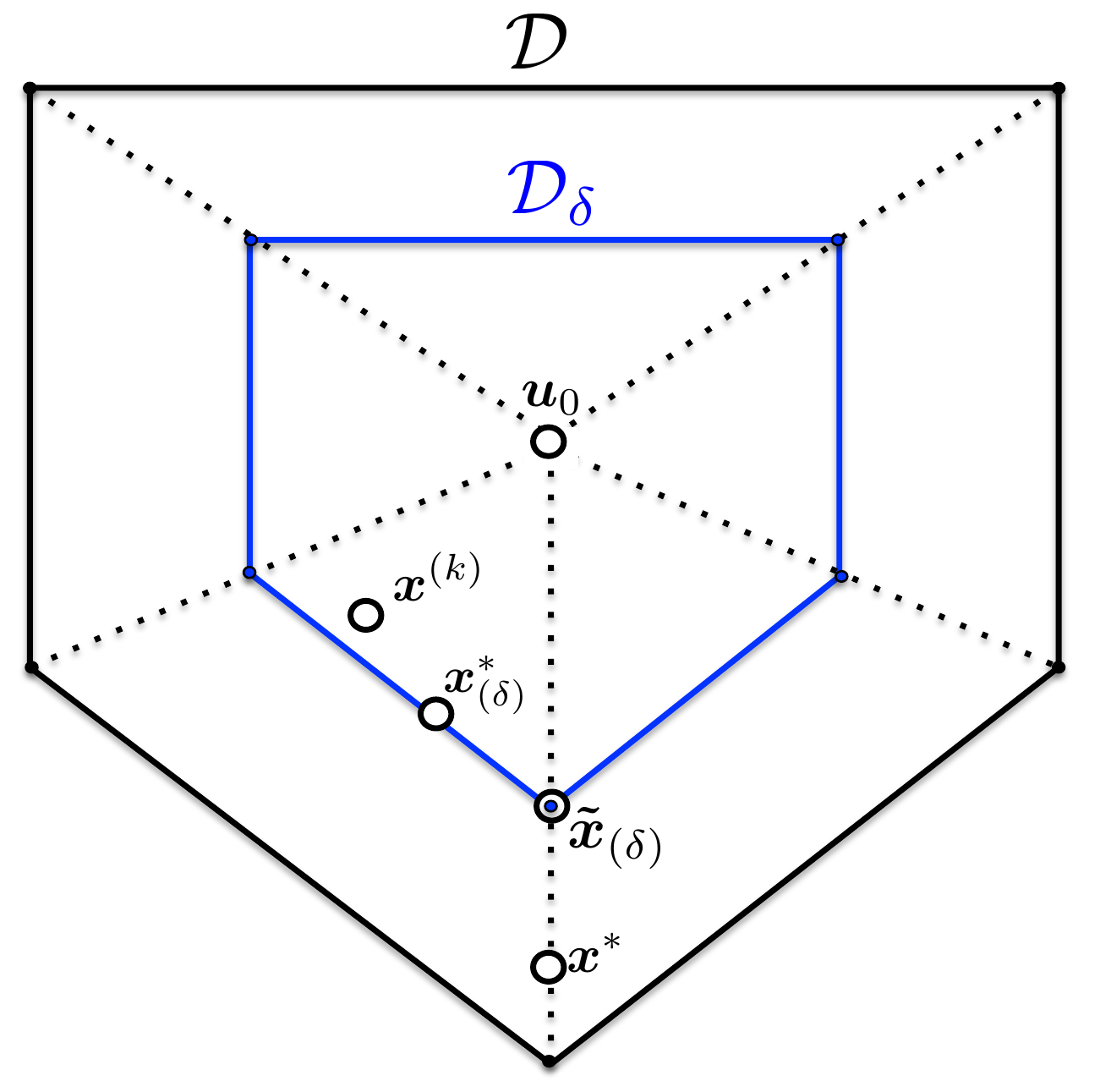}
\vspace{-2mm}
\caption{Illustration of the four points considered for the error analysis of the fixed-$\shrinkAmount$ variant}
\label{fig:fixed_eps_diagram}
\vspace{-3mm}
\end{figure}

We note that the modulus of continuity function $\modContinuity$ quantifies the level of continuity of $f$. For a Lipschitz continuous function, we have $\modContinuity(\sigma) \leq L \sigma$. If instead we have $\modContinuity(\sigma) \leq C \sigma^\alpha$ for some $\alpha \in [0,1]$, then $f$ is actually $\alpha$-H\"{o}lder continuous. We will see in Section~\ref{sec:trw_weak_lip} that the TRW objective is not Lipschitz continuous, but it is  $\alpha$-H\"{o}lder continuous for any $\alpha < 1$, and so is ``almost'' Lipschitz continuous. From the theorem, we see that to get an accuracy of the order $\stopCrit$, we would need $(\shrinkAmount \diam(\DOM))^\alpha < \stopCrit$, and thus a contraction of $\shrinkAmount < \frac{\stopCrit^{(1/\alpha)}}{\diam(\DOM)}$.

\section{Convergence with Adaptive-$\shrinkAmount$ \label{sec:theory_adaptive_eps}}

In this section, we show the convergence of the adaptive-$\shrinkAmount$ FW algorithm
to optimize a function $f$ satisfying the properties in Problem~\ref{prop:generic_fxn} and Property~\ref{prop:prop_bounded_grad} (Lipschitz gradient over $\DOMeps$ with bounded growth).

The adaptive update for $\shrinkAmount$ (given in Algorithm~\ref{alg:adaptive_update}) can
be used with the standard Frank-Wolfe optimization algorithm or also the fully corrective Frank-Wolfe (FCFW) variant. In FCFW, we ensure that every update makes more progress than a standard FW step with line-search, and thus we will show the convergence result in this section for standard FW (which also applies to FCFW). We describe the FCFW variant with approximate correction steps in Algorithm~\ref{alg:adaptive_eps}, as this is what we used in our experiments.

We first list a few definitions and lemmas that will be used for the main convergence convergence result given in Theorem~\ref{thm:convergence_adaptive_eps}.
We begin with the definitions of duality gaps that we use throughout this section. The Frank-Wolfe 
gap is our primary criterion for halting and measuring the progress of the optimization over $\DOM$.
The uniform gap is a measure of the decrease obtainable from moving towards the uniform distribution.

\begin{definition}
	\label{prop:g_k_properties_true}
	We define the following gaps:
	\begin{enumerate}
	\item The Frank-Wolfe (FW) gap is defined as:
	$\gk :=  \brangle{-\nabla f(\xk),\sk-\xk}$. 

	\item The uniform gap is defined as:
	$\gunif := \brangle{-\nabla f(\xk),\unif-\xk}$. 
	
	\item The FW gap over $\DOMeps$ is:
	$\gepsk{k} := \brangle{-\nabla f(\xk),\skeps-\xk}$. 
	\end{enumerate}
\end{definition}

The name for the uniform gap comes from the fact that the FW gap over $\DOMeps$ can be expressed as a convex combination of the FW gap over $\DOM$ and the uniform gap:
\begin{align}
\gepsk{k} &=  \brangle{-\nabla f(\xk), \,\, (1-\epsk{k}) \sk + \epsk{k} \unif-\xk} \nonumber \\
		  &=  (1-\shrinkAmount^{(k)})\gk + \shrinkAmount^{(k)} \gunif. \label{eq:gap_delta_decomposition}
\end{align}
The uniform gap represents the negative directional derivative of $f$ at $\xk$ in the direction $\unif - \xk$. When the uniform gap is negative (thus $f$ is increasing when moving towards $\unif$ from $\xk$), then the contraction is hurting progress, which explains the type of adaptive update for $\shrinkAmount$ given by Algorithm~\ref{alg:adaptive_update} where we consider shrinking $\shrinkAmount$ in this case. This enables us to crucially relate the FW gap over $\DOMeps$ with the one over $\DOM$, as given in the following lemma, using the assumption that $\shrinkAmount^{(\mathrm{init})} \leq \frac{1}{4}$.
\begin{lemma}[Gaps relationship]
	\label{prop:g_k_properties_adaptive}
	For iterates progressing as in Algorithm~\ref{alg:adaptive_eps} with adaptive update on $\shrinkAmount$ as given in Algorithm~\ref{alg:adaptive_update}, the gap over $\DOMeps$ and $\DOM$ are related as : $\gepsk{k}\geq\frac{\gk}{2}$.
\end{lemma}
\begin{proof}
	The duality gaps $\gk$ and $\gepsk{k}$ computed as defined in~\eqref{prop:g_k_properties_true} during Algorithm~\ref{alg:adaptive_eps} are related by equation~\eqref{eq:gap_delta_decomposition}.
	
We analyze two cases separately:

(1)	When $\gunif\geq 0$, for $\shrinkAmount^{(\mathrm{init})}\leq\frac{1}{4}$, we have $\gepsk{k} \geq \frac{3}{4}\gk$ as $\epsk{k} \leq \shrinkAmount^{(\mathrm{init})}$.

(2)	When $\gunif <0$, from the update rule in lines~5 to~7 in Algorithm~\ref{alg:adaptive_update}, we have
	$\epsk{k}\leq \frac{\gk}{-4\gunif}\implies \epsk{k}\gunif\geq-\frac{\gk}{4}$. Therefore, $\gepsk{k} \geq \frac{3}{4}\gk-\frac{\gk}{4}  = \frac{\gk}{2}$.

Therefore, the gap over $\DOMeps$ and $\DOM$ are related as : $\gepsk{k}\geq\frac{\gk}{2}$.
\end{proof}

Another property that we will use in the convergence proof is that $-\gunifNox$ is upper bounded for any convex function $f$:\footnote{Note that on the other hand, $\gunifNox(\x)$ might go to infinity as $\x$ gets close to the boundary of $\DOM$ as the gradient of $f$ is allowed to be unbounded. Fortunately, we only need an upper bound on $-\gunifNox$, not a lower bound.}
\begin{lemma}[Bounded negative uniform gap]
\label{prop:prop_bounded_gunif}
Let $f$ be a continuously differentiable convex function on the relative interior of $\DOM$. Then for any fixed $\unif$ in the relative interior of $\DOM$, $\exists \gunifBound$
s.t. 
\begin{equation} 
\forall \x\in\DOM, \;\; -\gunifNox(\x) = \innerProd{\nabla f(\x)}{\unif-\x} \leq \gunifBound.
\end{equation}
In particular, we can take the finite value:
\begin{equation} \label{eq:uniformBound}
	B := \| \nabla f(\unif) \|_* \diam_{\| \cdot \|} (\DOM)
\end{equation}
\end{lemma}
\begin{proof}
As $f$ is convex, its directional derivative is a monotone increasing function in any direction. Let $\unif$ and $\x$ be points in the relative interior of $\DOM$; then their gradient exists and we have by the monotonicity property:
\begin{align*}
\innerProd{\nabla f(\unif) - \nabla f(\x)}{\unif-\x} \geq 0 \\
\implies \innerProd{\nabla f(\unif)}{\unif-\x} \geq \innerProd{\nabla f(\x)}{\unif-\x} . 
\end{align*}
This inequality is valid for all $\x$ in the relative interior of $\DOM$, and can be extended to the boundary by taking limits (with potentially the RHS become minus infinity, but this is not a problem). Finally, by the definition of the dual norm (generalized Cauchy-Schwartz), we have $ \innerProd{\nabla f(\unif)}{\unif-\x} \leq \| \nabla f(\unif) \|_* \| \unif - \x \| \leq \| \nabla f(\unif) \|_*  \diam_{\| \cdot \|} (\DOM)$.
\end{proof}

Finally, we need a last property of Algorithm~\ref{alg:adaptive_eps} that allows us to bound the amount of perturbation $\epsk{k}$ of the polytope
at every iteration as a function of the sub-optimality over $\DOM$.  
\begin{lemma}[Lower bound on perturbation]
	\label{prop:eps_k_properties}
	Let $B$ be a bound such that $-8\gunifNox(\x)\leq \gunifBound$ for all $\x \in \DOM$ (given by Lemma~\ref{prop:prop_bounded_gunif}). Then at every stage of Algorithm~\ref{alg:adaptive_eps}, we have that:
	$$
		\epsk{\mathrm{init}} \geq \epsk{k}\geq\min\left\{\frac{\hk}{\gunifBound},\epsk{\mathrm{init}}\right\},
    $$
	where $\epsk{\mathrm{init}}$ is the initial value of $\shrinkAmount$ and $\hk := f(\xk)-f(\x^*)$ is the sub-optimality of the iterate.
\end{lemma}

\begin{proof}
	When defining $\epsk{k}$ in step 10 of Algorithm~\ref{alg:adaptive_eps}, we either preserve the value of $\epsk{k-1}$ or if we update it, then by the lines~6 and~7 of Algorithm~\ref{alg:adaptive_update}, we have $\epsk{k} \geq \frac{\tilde{\shrinkAmount}}{2} = \frac{1}{2} \frac{\gk}{-4\gunif} \geq  \frac{\gk}{\gunifBound}$ (by using $\gunif < 0$ in this case). Since $\gk\geq\hk$ (the FW gap always upper bounds the suboptimality), we conclude $\epsk{k} \geq \min\{\frac{h_k}{\gunifBound}, \epsk{k-1}\}$. 
	Unrolling this recurrence, we thus get:
	$$\epsk{k} \geq \min\left\{ \min_{0\leq l \leq k} \frac{\hkarg{l}}{\gunifBound}, \, \, \epsk{\mathrm{init}}\right\}
	=  \min\left\{ \frac{\hk}{\gunifBound}, \epsk{\mathrm{init}}\right\}. $$
	For the last equality, we used the fact that $h_k$ is non-increasing since Algorithm~\ref{alg:adaptive_eps} decreases the objective at every iteration (using the line-search in step~14).
\end{proof}

We now bound the generalization of a standard recurrence that will arise in the proof of convergence. This is a generalization of the technique used in~\citetsup{teo2007erm} (also used in the context of Frank-Wolfe in the proof of Theorem C.4 in \citet{lacoste2012block}). The basic idea is that one can bound a recurrence inequality by the solution to a differential equation. 
We provide a detailed proof of the 
bound for completeness here. 
\begin{lemma}[Recurrence inequality solution]
	\label{lem:differential_bound}
	Let $1< a \leq b$. Suppose that $\h_k$ is any non-negative sequence that satisfies the recurrence inequality: 
	$$\h_{k+1}\leq \h_k-\frac{1}{b C_0} (\h_k)^a \qquad \text{with initial condition} \quad\h_0^{a-1}\leq C_0.$$
	Then $\h_k$ is strictly decreasing (unless it equals zero) and can be bounded for $k \geq 0$ as: 
	$$\h_{k}\leq \left(\frac{C_0}{(\frac{a-1}{b})k+1}\right)^{\frac{1}{a-1}}$$
\end{lemma}

\begin{proof}
Taking the continuous time analog of the recurrence inequality, we consider
the differential equation:
$$\frac{\del h}{\del t} = \frac{-h^a}{bC_0} \qquad \text{with initial condition} \quad h(0)=C_0^{\frac{1}{a-1}}.$$
Solving it:
	\begin{equation*}
		\begin{split}
		&\frac{\del h}{\del t} = \frac{-h^a}{b C_0}\\
		\implies &\int \frac{\del h}{h^a} = \int \frac{-\del t}{b C_0}\\
		\implies &\left[\frac{-h^{1-a}}{a-1}\right]^{h(t)}_{h(0)} = -\frac{t-0}{b C_0}\\
		&\algComment{ Using the initial conditions:}\\
		\implies &\frac{-1}{h(t)^{a-1}} + \frac{1}{C_0} = \frac{-t(a-1)}{b C_0}\\
		\implies &\frac{1}{h(t)^{a-1}} = \left( (\frac{a-1}{b}) t + 1\right)\frac{1}{C_0}\\
		\implies &h(t) = \left(\frac{C_0}{ (\frac{a-1}{b}) t +1}\right)^{\frac{1}{a-1}}.
		\end{split}
	\end{equation*}
We now denote the solution to the differential equation as $\tilde{h}(t)$. Note that it is a strictly decreasing convex function 
(which could also be directly implied from the differential equation as: 
$\frac{\del^2 h}{\del t^2} = -a\underbrace{\frac{h^{a-1}}{b C_0}}_{>0} \underbrace{h'(t)}_{<0}>0$ ).
Our strategy will be to show by induction that if $\h_k\leq \tilde{h}(k)$, then $\h_{k+1}\leq \tilde{h}(k+1)$. This allows us to bound the recurrence by the solution to the differential equation. 

Assume that $\h_k\leq \tilde{h}(k)$. The base case is $\h_0\leq\tilde{h}(0)= C_0^{\frac{1}{a-1}}$, which is true by the initial condition on $\h_0$.

Consider the utility function $l(h) := h-\frac{h^a}{b C_0}$ which is maximized at $\bar{h} := \left(\frac{b C_0}{a}\right)^{\frac{1}{a-1}}$. This function can be verified to be strictly concave for $a > 1$ and therefore is increasing for $h\leq \bar{h}$. Note that the recurrence inequality can be written as $\h_{k+1} \leq l(\h_k)$.
Since $\tilde{h}$ is decreasing and that $\tilde{h}(0)) = C_0^{\frac{1}{a-1}} \leq \left(\frac{bC_0}{a}\right)^{\frac{1}{a-1}} = \bar{h}$ (the last inequality holds since $b\geq a$), we have $\tilde{h}(t) \leq \bar{h}$ for all $t \geq 0$, and so $\tilde{h}(t)$ is always in the monotone increasing region of $l$. 

From the induction hypothesis and the monotonicity of $l$, we thus get that $l(\h_k)\leq l(\tilde{h}(k))$. 

Now the convexity of $\tilde{h}(t)$ gives us $\tilde{h}(k+1)\geq \tilde{h}(k)+\tilde{h}'(k)=\tilde{h}(k)-\frac{\tilde{h}(k)^a}{b C_0} = l(\tilde{h}(k))$.
Combining these two facts with the recurrence inequality $\h_{k+1} \leq l(\h_k)$, we get:
$\h_{k+1}\leq l(\h_k) \leq l(\tilde{h}(k))\leq \tilde{h}(k+1)$, completing the induction step
and the main part of the proof.

Finally, whenever $\h_k > 0$, we have that $\h_{k+1} < \h_k$ from the recurrence inequality, and so 
$\h_k$ is strictly decreasing as claimed.
\end{proof}

Given these elements, we are now ready to state the main convergence result for Algorithm~\ref{alg:adaptive_eps}. The convergence rate goes through three stages with increasingly slower rate. The level of suboptimality $\hk$ determines the stage. We first give the high level intuition behind these stages. Recall that by Lemma~\ref{prop:eps_k_properties}, $\hk$ lower bounds the amount of perturbation $\epsk{k}$, and thus when $\hk$ is big, the function $f$ is well-behaved by Property~\ref{prop:prop_bounded_grad}. In the first stage, the suboptimality is bigger than some target constant (which implies that the FW gap is big), yielding a geometric rate of decrease of error (as is standard for FW with line-search in the first few steps). In the second stage, the suboptimality is in an intermediate regime: it is smaller than the target constant, but big enough compared to the initial $\shrinkAmount^{\mathrm{init}}$ so that $f$ is still well-behaved on $\DOM_{\epsk{k}}$. We get there the usual $O(1/k)$ rate as in standard FW. Finally, in the third stage, we get the slower $O(k^{-\frac{1}{p+1}})$ rate where the growth in $O(\shrinkAmount^{-p})$ of the Lipschitz constant of $f$ over $\DOMeps$ comes into play.
\begin{theorem}[Global convergence for adaptive-$\shrinkAmount$ variant over $\DOM$]
	\label{thm:convergence_adaptive_eps}
	Consider the optimization of $f$ satisfying the properties in Problem~\ref{prop:generic_fxn} and Property~\ref{prop:prop_bounded_grad}.  Let~$\Cconst := L \diam_{\|\cdot\|}(\DOM)^2$, where $L$ is from Property~\ref{prop:prop_bounded_grad}. Let~$B$ be the upper bound on the negative uniform gap:~$-8\gunifNox(\x)\leq \gunifBound$ for all $\x \in \DOM$, as used in Lemma~\ref{prop:eps_k_properties} (arising from Lemma~\ref{prop:prop_bounded_gunif}). Then the iterates~$\xk$ obtained by running the Frank-Wolfe updates over $\DOMeps$ with line-search with $\shrinkAmount$ updated according to Algorithm~\ref{alg:adaptive_update} (or as summarized in a FCFW variant in Algorithm~\ref{alg:adaptive_eps}), have suboptimality~$h_k$ upper bounded as:
	\begin{enumerate}
		\item $\hk\leq \left(\frac{1}{2}\right)^{k}\h_{0} +\frac{\Cconst}{\shrinkAmount_0^p}$ for $k$ such that $\hk \geq \max\{ B\shrinkAmount_0, \frac{2\Cconst}{\shrinkAmount_0^p}\}$, \\
		\item $\hk\leq\frac{2\Cconst}{\shrinkAmount_0^p}\left[\frac{1}{\frac{1}{4} (k-k_0)+1}\right]$ for $k$ such that $\gunifBound\shrinkAmount_0\leq\hk\leq\frac{2\Cconst}{\shrinkAmount_0^p}$,\\
		\item $\hk\leq \left[\frac{\max(\Cconst,\gunifBound\shrinkAmount_0^{p+1})\gunifBound^p}{\frac{p+1}{\max(8,p+2)}(k-k_1)+1}\right]^{\frac{1}{p+1}} = O(k^{-\frac{1}{p+1}})$ for $k$ such that $\hk\leq\gunifBound\shrinkAmount_0$,\\
	\end{enumerate}
	where $\shrinkAmount_0 = \epsk{\mathrm{init}}$, $h_0$ is the initial suboptimality, and $k_0$ and $k_1$ are the number of steps to reach stage~2 and~3 respectively which are bounded as: $k_0\leq \max(0,  \lceil\log_{\frac{1}{2}}\frac{\Cconst}{ h_0 \shrinkAmount_0^p}\rceil )$, 
	$k_0 \leq k_1\leq k_0 + \max\left(0,\lceil \frac{8\Cconst}{\gunifBound\shrinkAmount_0^{p+1}}\rceil -4\right)$.
	
\end{theorem}
\begin{proof}
	Let $\x_\stepsize := \xk + \stepsize \dd_k^\FW$ with $\dd_k^\FW$ defined in step~12 in Algorithm~\ref{alg:adaptive_eps}. Note that $\x_\stepsize \in \DOMeps$ with $\shrinkAmount = \epsk{k}$ for all $\stepsize \in [0,1]$. We apply the Descent Lemma~\ref{lem:descent_lemma} on this update to get:
	$$f(\x_\stepsize) \leq f(\xk) + \stepsize\brangle{\nabla f(\xk),\dd_k^\FW} + \stepsize^2\frac{L \|\dd_k^\FW\|^2}{2\left(\epsk{k}\right)^p} \qquad \forall \stepsize \in [0,1].$$
	We have $L \| \dd_k^\FW \|^2 \leq \Cconst$ by assumption and $\brangle{\nabla f(\xk),\dd_k^\FW} = -\geps$ by definition. Moreover, $\xnext$ is defined to make at least as much progress than the line-search result
	$\min_{\stepsize \in [0,1]} f(\x_\stepsize)$ (line~14 and~15), and so we have:
	\begin{align*}
		f(\xnext) &\leq f(\xk) - \stepsize\geps + \stepsize^2\frac{\Cconst}{2 \left(\epsk{k}\right)^p} \quad \forall \stepsize \hiderel{\in} [0,1] \\
		&\leq f(\xk) - \frac{\stepsize}{2}\gk + \stepsize^2\frac{\Cconst}{2 \left(\epsk{k}\right)^p} \quad \forall \stepsize \hiderel{\in} [0,1].
	\end{align*}
	For the final inequality, we used Lemma~\ref{prop:g_k_properties_adaptive} which relates the gap over $\DOMeps$ to the gap over $\DOM$.

	Subtracting $f(\xopt)$ from both sides and using $\gk\geq\hk$ by convexity, we get:
	$$\hnext\leq\hk-\frac{\stepsize\hk}{2} + \frac{\stepsize^2\Cconst}{2 \left(\epsk{k}\right)^p}.$$

	Now, using Lemma~\ref{prop:eps_k_properties}, we have that $\epsk{k}\geq\min(\frac{\hk}{\gunifBound},\epsk{\mathrm{init}})$:
	\begin{align}
		\label{eqn:master_eqn}
		\hnext &\leq\hk-\stepsize\frac{\hk}{2} + \frac{\stepsize^2}{2}\frac{\Cconst}{\left(\min(\frac{\hk}{\gunifBound},\epsk{\mathrm{init}})\right)^p} \quad \forall \stepsize \hiderel{\in} [0,1].
	\end{align}

	We refer to \eqref{eqn:master_eqn} as the master inequality. Since we no longer have a dependance on $\epsk{k}$, we refer to $\epsk{\mathrm{init}}$ as $\shrinkAmount_0$. 
	We now follow a similar form of analysis as in the proof of Theorem~C.4 in~\citet{lacoste2012block}.
	To solve this and bound the suboptimality, we consider three stages:
	\begin{enumerate}
		\item \textbf{Stage 1:} The $\min$ in the denominator is $\shrinkAmount_0$ and $h_k$ is big: $\hk\geq \max\{ B\shrinkAmount_0, \frac{2\Cconst}{\shrinkAmount_0^p}\}$.
		\item \textbf{Stage 2:} The $\min$ in the denominator is $\shrinkAmount_0$ and $h_k$ is small: $\gunifBound\shrinkAmount_0\leq\hk\leq\frac{2\Cconst}{\shrinkAmount_0^p}$.
		\item \textbf{Stage 3:} The $\min$ in the denominator is $\frac{\hk}{\gunifBound}$, i.e.: $\hk\leq\gunifBound\shrinkAmount_0$.
	\end{enumerate}

	Since $\hk$ is decreasing, once we leave a stage, we no longer re-enter it. 
	The overall strategy for each stage is as follows. For each recurrence that we get, we select
	a $\stepsize^*$ that realizes the tightest upper bound on it. 

	Since we are restricted that 
	$\stepsize^*\in[0,1]$, we have to consider when $\stepsize^*>1$ and $\stepsize^*\leq 1$. For the former,
	we bound the recurrence obtained by substituting $\stepsize=1$ into \eqref{eqn:master_eqn}. For the latter, 
	we substitute the form of $\stepsize^*$ into the recurrence and bound the result.

	\subsection*{Stage 1}
	We consider the case where $\hk\geq\gunifBound\shrinkAmount_0$. This yields: 
	\begin{align}
		\label{eqn:case_1}
		\hnext &\leq\hk - \frac{\stepsize\hk}{2} + \frac{\stepsize^2 \Cconst}{2 \left(\shrinkAmount_0\right)^p}	
	\end{align}
	The bound is minimized by setting $\stepsize^* = \frac{\hk\shrinkAmount_0^p}{2\Cconst}$. On the other hand, the bound is only valid for $\stepsize \in [0,1]$, and thus if $\stepsize^* > 1$, i.e. $\hk > \frac{2\Cconst}{\shrinkAmount_0^p}$ (stage 1), then $\stepsize = 1$ will yield the minimum feasible value for the bound. Unrolling the recursion~\eqref{eqn:case_1} for $\stepsize = 1$ during this stage (where $\h_l > \frac{2\Cconst}{\shrinkAmount_0^p}$ for $l < k$ as $\h_k$ is decreasing), we get:
	
	\begin{align}
		\hnext &\leq \frac{\hk}{2} + \frac{\Cconst}{2\shrinkAmount_0^p} \nonumber \\	
	&\leq \frac{1}{2}\left(\frac{\hprev}{2} + \frac{\Cconst}{2\shrinkAmount_0^p}\right) + \frac{\Cconst}{2\shrinkAmount_0^p} \nonumber \\
	&\leq \left(\frac{1}{2}\right)^{k+1} \h_0 + \frac{\Cconst}{2\shrinkAmount_0^p}\underbrace{\sum_{l=0}^{k}\left(\frac{1}{2}\right)^l}_{\leq\sum_{l=0}^{\infty}\left(\frac{1}{2}\right)^{l} = 2} \nonumber \\
	\text{thus} \quad \h_k &\leq \left(\frac{1}{2}\right)^{k} \h_0 + \frac{\Cconst}{\shrinkAmount_0^p} \label{eqn:hk_stage1},
	\end{align}
	giving the bound for the iterates in the first stage.

	We can compute an upper bound on the number of steps it takes to reach a suboptimality of $\frac{2\Cconst}{\shrinkAmount_0^p}$ by looking at the minimum $k$ which ensures that the bound in~\eqref{eqn:hk_stage1} becomes smaller than $\frac{2\Cconst}{\shrinkAmount_0^p}$, yielding
	$k_{\max} = \max(0,\lceil\log_{\frac{1}{2}}\frac{\Cconst}{\h_0\shrinkAmount_0^p}\rceil)$. Therefore, let $k_0\leq k_{\max}$ be the 
	first $k$ such that $\hk\leq\frac{2\Cconst}{\shrinkAmount_0^p}$.

	\subsection*{Stage 2}
	For this case analysis, we refer to $k$ as being the iterations \emph{after} $k_0$ steps have elapsed. I.e. if $k_{\mathrm{new}} := k-k_0$, then we refer to $k_{\mathrm{new}}$ as $k$ moving forward. 
	
	In stage~2, we suppose that $\gunifBound\shrinkAmount_0\leq\hk\leq\frac{2\Cconst}{\shrinkAmount_0^p}$. This means that $\stepsize^* = \frac{\hk\shrinkAmount_0^p}{2\Cconst}\leq 1$.
	
	Substituting $\stepsize=\stepsize^*$ into \eqref{eqn:case_1} yields: $\hnext\leq\hk -\hk^2\frac{\shrinkAmount_0^p}{8\Cconst}$.

	Using the result of Lemma~\ref{lem:differential_bound} with $a=2$, $b=4$ and $C_0 = \frac{2\Cconst}{\shrinkAmount_0^p}$, we get the bound:	
	$$\hk\leq\frac{\frac{2\Cconst}{\shrinkAmount_0^p}}{\frac{k-k_0}{4}+1}.$$
	
	It is worthwhile to point out at this juncture that the bound obtained for stage~$2$ is the same as the one for regular Frank-Wolfe, but with a factor of $4$ worse due to the factor of $\frac{1}{2}$ in front of the FW gap which appeared due to Lemma~\ref{prop:g_k_properties_adaptive}.

	\subsection*{Stage 3}
	Here, we suppose $\hk\leq\gunifBound\shrinkAmount_0$. We can compute a bound on the number of steps $k_1$ needed get to stage~3 by looking at the number of steps it takes for the bound in stage~2 to becomes less than $\gunifBound\shrinkAmount_0$:
	\begin{equation*}
		\begin{split}
		\frac{2\Cconst}{\shrinkAmount_0^p} \left[\frac{4}{k_1-k_0+4}\right] \leq \gunifBound\shrinkAmount_0\\
		\left[\frac{1}{k_1-k_0+4}\right] \leq \frac{\gunifBound\shrinkAmount_0^{p+1}}{8\Cconst}\\
		k_1 \geq k_0 + \lceil\frac{8\Cconst}{\gunifBound\shrinkAmount_0^{p+1}}\rceil - 4.
		\end{split}
	\end{equation*}

	As before, moving forward, our notation on $k$ represents the number of steps taken after $k_1$ steps. 
	
	Then, the master inequality~\eqref{eqn:master_eqn} becomes: 
		$$\hnext\leq \hk -\frac{\stepsize}{2}\hk + \frac{\stepsize^2\Cconst\gunifBound^p}{2\hk^p}.$$

	To simplify the rest of the analysis, we replace $\Cconst\gunifBound^p$ with $F := \max(\gunifBound\shrinkAmount_0^{p+1},\Cconst)\gunifBound^p$. We then get the bound:
	\begin{equation}
		\label{eqn:case_2_mod}
		\hnext\leq \hk-\frac{\stepsize}{2}\hk + \frac{\stepsize^2F}{2\hk},
	\end{equation}
	which is minimized by setting $\stepsize^*:=\frac{\hk^{p+1}}{2F}$.  
	Since $F\geq \gunifBound^{p+1}\shrinkAmount_0^{p+1}$ (by construction) and $\hk^{p+1}\leq (\gunifBound\shrinkAmount_0)^{p+1}$ (by the condition to be in stage~3),
	we necessarily have that $\stepsize^*\leq 1$. We chose the value of $F$ to avoid
	having to consider the possibility $\stepsize^* > 1$ as we did in the distinction between stage~1 and stage~2.
	
	Hence, substituting $\stepsize = \stepsize^*$ in~\eqref{eqn:case_2_mod}, we get:
	$$\hnext\leq\hk-\frac{\hk^{p+2}}{8F}.$$

	Using the result of Lemma~\ref{lem:differential_bound} with $a=p+2$, $b=\max(8,p+2)$ and $C_0 = F$, we get the bound:	
	$$\hk\leq\left[\frac{\max(\Cconst,\gunifBound \shrinkAmount_0^{p+1})\gunifBound^p}{\frac{p+1}{\max(8,p+2)}(k-k_1)+1}\right]^{\frac{1}{p+1}} = O(k^{-\frac{1}{p+1}}),$$
	concluding the proof.
\end{proof}

Interestingly, the obtained rate of $O(1/\sqrt{k})$ for $p=1$ (for the TRW objective e.g.) is the standard rate that one would get for the optimization of a general non-smooth convex function with the projected subgradient method (and it is even a lower bound for some class of first-order methods; see e.g. Section~3.2 in~\citetsup{nesterov2004lectures}). The fact that our function $f$ does not have Lipschitz continuous gradient on the whole domain brings us back to the realm of non-smooth optimization. It is an open question whether Algorithm~\ref{alg:adaptive_eps} has an optimal rate for the class of functions defined in the assumptions of Theorem~\ref{thm:convergence_adaptive_eps}.

\section{Properties of the TRW Objective \label{sec:trw_properties}}
In this section, we explicitly compute bounds for the constants appearing in the convergence statements for our fixed-$\shrinkAmount$ and adaptive-$\shrinkAmount$ algorithms for the optimization problem given by:
$$\min_{\vmu\in\MARG}-\TRW.$$ 
In particular, we compute the Lipschitz constant for its gradient over $\Meps$ (Property~\ref{prop:prop_bounded_grad}), we give a form for its modulus of continuity function $\omega(\cdot)$ (used in Theorem~\ref{thm:convergence_fixed_eps_main}), and we compute $B$, the upper bound on the negative uniform gap (as used in Lemma~\ref{prop:prop_bounded_gunif}).

\subsection{Property~\ref{prop:prop_bounded_grad} : Controlled Growth of Lipschitz Constant over $\Meps$} \label{sec:trw_bounded_lip}

We first motivate our choice of norm over $\MARG$. Recall that $\vmu$ can be decomposed into $|V| + |E|$ blocks, with one pseudo-marginal vector $\bmu_i \in \Delta_{\VAL_i}$ for each node $i \in V$, and one vector $\bmu_{ij} \in \Delta_{\VAL_i \VAL_j}$ per edge $\{i,j\} \in E$, where $\Delta_d$ is the probability simplex over $d$ values.
We let $c$ be the cliques in the graph (either nodes or edges). From its definition in~\eqref{eqn:trw_entropy_like_bethe}, $f(\vmu) := -\TRW$ decomposes as a separable sum of functions of each block only:
	\begin{equation} \label{eq:TRW_decompsed}
	f(\vmu) := -\TRW = -\sum_c \left(K_c H(\bmu_c) +\innerProd{\btheta_c}{\bmu_c}\right) =: \sum_c g_c(\bmu_c),
	\end{equation}
where $K_c$ is $(1- \sum_{j \in \nbrs{i}} \rho_{ij})$ if $c = i$ and $\rho_{ij}$ if $c = \{i,j\}$. The function $g_c$ also decomposes as a separable sum: 
\begin{equation} \label{eq:gc_function}
g_c(\bmu_c) := \sum_{x_c} K_c \mu_c(x_c) \log (\mu_c(x_c)) - \theta_c(x_c) \mu_c(x_c) =: \sum_{x_c} g_{c,x_c} (\mu_c(x_c)). 
\end{equation}
As $\MARG$ is included in a product of probability simplices, we will use the natural $\ell_\infty/\ell_1$ block-norm, i.e. $\| \vmu \|_{\infty,1} := \max_{c} \|\bmu_c \|_1$. The diameter of $\MARG$ in this norm is particularly small: $\diam_{\|\cdot \|_{\infty,1}}(\MARG) \leq 2$. The dual norm of the $\ell_\infty/\ell_1$ block-norm is the $\ell_1/\ell_\infty$ block-norm, which is what we will need to measure the Lipschitz constant of the gradient (because of the dual norm pairing requirement from the Descent Lemma~\ref{lem:descent_lemma}).

\begin{lemma}
	\label{lem:trw_bounded_lip_eps}
	Consider the $\ell_\infty/\ell_1$ norm on $\MARG$ and its dual norm $\ell_1/\ell_\infty$ to measure the gradient. Then
$\nabla \TRW$ is Lipschitz continuous over $\Meps$ with respect to these norms with Lipschitz constant 
$L_{\shrinkAmount} \leq \frac{L}{\shrinkAmount}$ with:
\begin{equation} \label{eq:TRW_L_constant} 
L \leq 4 |V| \,  \max_{ij \in E} \,\, (\VAL_i\VAL_j).
\end{equation}
\end{lemma}
\begin{proof}
	We first consider one scalar component of the separable $g_c(\bmu_c)$ function given in~\eqref{eq:gc_function} (i.e. for one $\mu_c(x_c)$ coordinate). Its derivative is $K_c (1 + \log (\mu_c(x_c)) - \theta_c(x_c)$ with second derivative $\frac{K_c}{ \mu_c(x_c)}$. If $\vmu \in \Meps$, then we have $\mu_c(x_c) \geq \shrinkAmount u_0(x_c) = \frac{\shrinkAmount}{n_c}$, where $n_c$ is the number of possible values that the assignment variable $x_c$ can take. Thus for $\vmu \in \Meps$, we have that the $x_c$-component of $g_c$ is Lipschitz continuous with constant $|K_c| n_c / \shrinkAmount$. We thus have:
	\begin{align*}
	\| \nabla g_c(\bmu_c) - \nabla g_c(\bmu'_c) \|_\infty &= \max_{x_c} \,\, | g'_{c, x_c} (\mu(x_c)) - g'_{c, x_c} (\mu'(x_c)) | \\
	&\leq \frac{|K_c| n_c}{\shrinkAmount} \|\bmu_c - \bmu_c' \|_\infty \leq \frac{|K_c| n_c}{\shrinkAmount} \|\bmu_c - \bmu_c' \|_1.
	\end{align*}
	Considering now the $\ell_1$-sum over blocks, we have:
	\begin{align*}
	\|\nabla f(\vmu) - \nabla f(\vmu') \|_{1,\infty} &= \sum_c \| \nabla g_c(\bmu_c) - \nabla g_c(\bmu'_c) \|_\infty  \\  
	&\leq \sum_c \frac{K_c n_c}{\shrinkAmount} \|\bmu_c - \bmu_c' \|_1 \leq \frac{1}{\shrinkAmount} \left(\sum_c K_c n_c \right) \| \vmu - \vmu' \|_{\infty, 1}.
	\end{align*}
	The Lipschitz constant is thus indeed $\frac{L}{\shrinkAmount}$ with $L := \sum_c |K_c| n_c$.
	Let us first consider the sum for $c \in V$; we have $K_i = 1 - \sum_{j \in \nbrs{i}} \rho_{ij}$. 
	Thus:
	\begin{align*}
	\sum_i |K_i| &\leq |V| + \sum_i \sum_{j \in \nbrs{i}} \rho_{ij} \\
				&= |V| + 2 \sum_{ij \in E} \rho_{ij} = |V| + 2 (|V| - 1) \leq 3 |V|. 	
	\end{align*}
	Here we used the fact that $\rho_{ij}$ came from the marginal probability of edges of spanning trees (and so with $|V|-1$ edges). Similarly, we have $\sum_{ij \in E} |K_{ij}| \leq |V|$. Combining these we get:
	\begin{align} \label{eq:Kc_bound}
	L = \sum_c |K_c| n_c \leq (\max_c n_c) \sum_c |K_c| \leq \max_{ij \in E} \VAL_i \VAL_j 4 |V|.
	\end{align}
\end{proof}

\begin{remark}
The important quantity in the convergence of Frank-Wolfe type algorithms is $\tilde{C} = L \diam(\MARG)^2$. We are free to take any dual norm pairs to compute this quantity, but some norms are better aligned with the problem than others. Our choice of norm in Lemma~\ref{lem:trw_bounded_lip_eps} gives $\tilde{C} \leq 16 |V| k^2$ where $k$ is the maximum number of possible values a random variable can take. It is interesting that $|E|$ does not appear in the constant. If instead we had used the $\ell_2/\ell_1$ block-norm on $\MARG$, we get that $\diam_{\ell_2/\ell_1}(\MARG)^2 = 4 (|V| + |E|)$, while the constant $L$ with dual norm $\ell_2/\ell_\infty$ would be instead $\max_c |K_c| n_c$ which is bigger than $\max_c n_c = k^2$, thus giving a worse bound. 
\end{remark}

\subsection{Modulus of Continuity Function} \label{sec:trw_weak_lip}
We begin by computing a modulus of continuity function for $-x \log x$ with an additive linear term.

\begin{lemma}
  \label{lem:entropy_lipschitz} Let $g(x) := -Kx\log x + \theta x$. Consider $x, x' \in [0,1]$ such that $|x-x'| \leq\sigma$, then:
  \begin{equation} \label{eqn:modulus_g}
  |g(x')-g(x)|\leq \sigma|\theta| + 2 \sigma |K| \max\{-\log(2\sigma),1\} =: \omega_g(\sigma).
  \end{equation}
\end{lemma}
\begin{proof}
	Without loss of generality assume $x'>x$, then we have two cases:

	\textbf{Case i.} 
	If $x>\sigma$, then we have that the Lipschitz constant of $g(x)$ is $L_{\sigma} = |\theta| + |K||(1+\log\sigma)|$ (obtained by taking 
	the supremum of its derivative). Therefore, we have that $|g(x')-g(x)|\leq L_{\sigma}\sigma$. Note that $L_\sigma\sigma\to 0$ when $\sigma\to 0$ even if
	$L_\sigma\to\infty$, since $L_\sigma$ grows logarithmically.

	\textbf{Case ii.}
	If $x\leq\sigma$, then $x'\leq x + \sigma \leq 2\sigma$. Therefore: 
	\begin{equation} \label{eqn:g_inequal}
	|g(x')-g(x)| \leq |K||x\log x-x'\log x'| + |\theta||x'-x|.
	\end{equation}
	Now, we have that $-x\log x$ is non-negative for $x\in [0,1]$. Furthermore, we have that 
	$-x\log x$ is increasing when $x<\exp(-1)$ and decreasing afterwards.
	First suppose that $2 \sigma \leq \exp(-1)$; then $-x'\log x' \geq -x\log x \geq 0$ which implies:
	$$
	|x\log x-x'\log x'| \leq -x' \log x' \leq - 2\sigma \log (2\sigma).
	$$
	In the case $2 \sigma > \exp(-1)$, then we have:
	$$ |x\log x-x'\log x'| \leq \max_{y \in [0,1]} \{-y \log y\} = \exp(-1) \leq 2 \sigma.
	$$
	Combining these two possibilities, we get:
	$$
	 |x\log x-x'\log x'| \leq 2\sigma \max\{- \log (2\sigma), 1\}.
	$$
	The inequality~\eqref{eqn:g_inequal} thus becomes:
	$$|g(x')-g(x)| \leq |K|2 \sigma  \max\{- \log (2\sigma), 1\} + |\theta|\sigma,$$
	which is what we wanted to prove.
\end{proof}
For small $\sigma$, the dominant term of the function $\omega_g(\sigma)$ in Lemma~\ref{lem:entropy_lipschitz} is of the 
form $C\cdot-\sigma\log\sigma$ for a constant~$C$. If we require that this be smaller than some small $\xi > 0$, then we can choose an approximate 
$\sigma$ by solving for $x$ in $-Ax\log x = \xi$ yielding $x = \exp( W_{-1} \frac{\xi}{A})$ where $W_{-1}$ is the negative 
branch of the Lambert W-function. This is almost linear and yields approximately $x = O(\xi)$ for small $\xi$. In fact, we have that $\omega_g(\sigma) \leq C' \sigma^\alpha$ for any $\alpha < 1$, and thus $g$ is ``almost'' Lipschitz continuous.

\begin{lemma}
  \label{lem:trw_obj_weak_lipschitz}
  The following function is a modulus of continuity function for the $\TRW$ objective over $\MARG$ with respect to the $\ell_\infty$ norm:
  \begin{equation}
   \omega(\sigma) := \sigma\| \theta \|_1 + 2 \sigma \tilde{K} \max\{-\log(2\sigma),1\},
   \end{equation}
   where $\tilde{K} := 4 |V| \max_{ij \in E} \VAL_i \VAL_j$.
   
   That is, for $\vmu, \vmu' \in \MARG$ with $\| \vmu' - \vmu \|_\infty \leq \sigma$, we have:
   $$ | \text{TRW}(\vmu;\vtheta,\vrho) - \text{TRW}(\vmu';\vtheta,\vrho) | \leq \omega(\sigma) .$$
\end{lemma}
\begin{proof}
$\TRW$ can be decomposed into 
functions of the form $-Kx\log x + \theta x$ (see~\eqref{eq:TRW_decompsed} and~\eqref{eq:gc_function}) and so we apply the Lemma~\ref{lem:entropy_lipschitz} element-wise. Let $c$ index the clique component in the marginal vector. 
\begin{align*}
	| \text{TRW}(\vmu;\vtheta,\vrho) - \text{TRW}(\vmu';\vtheta,\vrho) | &= 
	\sum_c \sum_{x_c} | g_{c,x_c}(\mu_c(x_c)) - g_{c,x_c}(\mu'_c(x_c))| \\
	&\algComment{Using Lemma \ref{lem:entropy_lipschitz} and $\| \vmu' - \vmu \|_\infty \leq \sigma$}\\
	&\leq \sum_c \sum_{x_c} (|K_c|2 \sigma  \max\{- \log (2\sigma), 1\} + |\theta(x_c)|\sigma )\\
	&= 2\sigma  \max\{- \log (2\sigma), 1\} \sum_c |K_c|n_c +\|\theta\|_1\sigma,
\end{align*}
where we recall $n_c$ is the number of values that $x_c$ can take. By re-using the bound on $\sum_c |K_c| n_c$ from~\eqref{eq:Kc_bound}, we get the result.
\end{proof}

\subsection{Bounded Negative Uniform Gap} \label{sec:trw_bounded_unif}

\begin{lemma}[Bound for the negative uniform gap of TRW objective]
	\label{lem:bounded_gu0}
For the negative TRW objective $f(\vmu) := -\TRW$, the bound $B$ on the negative uniform gap as given in Lemma~\ref{prop:prop_bounded_gunif} for $\unif$ being the uniform distribution can be taken as:
\begin{equation} \label{eq:uniformBoundTRW}
	B =  2\sum_{c} \max_{x_c} |\theta_c(x_c)| =: 2 \|\vtheta \|_{1, \infty}
\end{equation}
\end{lemma}
\begin{proof}
	From Lemma~\ref{prop:prop_bounded_gunif}, we want to bound 
	$\| \nabla f(\unif) \|_* = \| \vtheta + \nabla_{\vmu} H(\unif;\vrho)) \|_*$. The clique entropy terms $H(\bmu_c)$ are maximized by the uniform distribution, and thus $\unif$ is a stationary point of the TRW entropy function with zero gradient. We can thus simply take $B = \|\vtheta\|_* \diam_{\| \cdot \|} (\MARG)$. By taking the $\ell_\infty / \ell_1$ norm on $\MARG$, we get a diameter of $2$, giving the given bound.
\end{proof}

\subsection{Summary}

We now give the details of suboptimality guarantees for our suggested algorithm to optimize $f(\vmu) := -\TRW$ over $\MARG$.
The (strong) convexity of the negative TRW objective is shown in \citep{wainwright2005new,london_icml15}. 
$\MARG$ is the convex hull of a finite number of vectors representing assignments to random variables and therefore a compact convex set. The entropy function is continously differentiable on the relative interior of the probability simplex, and thus the TRW objective has the same property on the relative interior of $\MARG$. Thus $-\TRW$ satisfies the properties laid out in Problem~\ref{prop:generic_fxn}.

\begin{lemma}[Suboptimality bound for optimizing $-\TRW$ with the fixed-$\shrinkAmount$ algorithm]
  \label{thm:convergence_fixed_eps_trw}
  For the optimization of $-\TRW$ over $\Meps$ with $\shrinkAmount\in(0,1]$, the suboptimality is bounded as: 
  
	\begin{equation} 
		\text{TRW}(\vmu^{*};\vtheta,\vrho) - \text{TRW}(\vmu^{(k)};\vtheta,\vrho) \leq \frac{2\mathcal{C}_{\shrinkAmount}}{(k+2)}+\modContinuity \left( 2\shrinkAmount \right), 
	\end{equation}
	 with $\vmu^*$ the optimizer of $\TRW$ in $\MARG$, 
	 where $C_{\shrinkAmount} \leq 16 \frac{|V|\max_{(ij)\in E}\VAL_i\VAL_j}{\shrinkAmount}$, 
	 and $\modContinuity(\sigma) = \sigma\|\vtheta\|_{1}+2\sigma \tilde{K}\max\{-\log (2\sigma),1\}$, where $\tilde{K} := 4 |V| \max_{ij \in E} \VAL_i \VAL_j$.
\end{lemma}
\begin{proof}
	Using $\diam_{\|\cdot \|_{\infty,1}}(\MARG) \leq 2$, and $L_\shrinkAmount$ from Lemma \ref{lem:trw_bounded_lip_eps},
	we can compute $C_{\shrinkAmount}\leq \diam(\MARG)^2L_{\shrinkAmount}$. 
	Lemma 
\ref{lem:trw_obj_weak_lipschitz} computes the modulus of continuity $\omega(\sigma)$.
The rate then follows directly from Theorem \ref{thm:convergence_fixed_eps}. 
\end{proof}

\begin{lemma}[Global convergence rate for optimizing $-\TRW$ with the adaptive-$\shrinkAmount$ algorithm]
	\label{thm:convergence_adaptive_eps_trw}
	Consider the optimization of $-\TRW$ over $\MARG$ with the optimum given by $\vmu^*$.
	The iterates~$\vmu^{(k)}$ obtained by running the Frank-Wolfe updates over $\Meps$ using line-search with $\shrinkAmount$ updated according to Algorithm~\ref{alg:adaptive_update} (or as summarized in a FCFW variant in Algorithm~\ref{alg:adaptive_eps}), have suboptimality~$h_k = \text{TRW}(\vmu^{*};\vtheta,\vrho) - \text{TRW}(\vmu^{(k)};\vtheta,\vrho)$ upper bounded as:
	\begin{enumerate}
		\item $\hk\leq \left(\frac{1}{2}\right)^{k}\h_{0} +\frac{\Cconst}{\shrinkAmount_0}$ for $k$ such that $\hk \geq \max\{ B\shrinkAmount_0, \frac{2\Cconst}{\shrinkAmount_0}\}$, \\
		\item $\hk\leq\frac{2\Cconst}{\shrinkAmount_0}\left[\frac{1}{\frac{1}{4} (k-k_0)+1}\right]$ for $k$ such that $\gunifBound\shrinkAmount_0\leq\hk\leq\frac{2\Cconst}{\shrinkAmount_0}$,\\
		\item $\hk\leq \left[\frac{\max(\Cconst,\gunifBound\shrinkAmount_0^{2})\gunifBound}{\frac{1}{4}(k-k_1)+1}\right]^{\frac{1}{2}} = O(k^{-\frac{1}{2}})$ for $k$ such that $\hk\leq\gunifBound\shrinkAmount_0$,\\
	\end{enumerate}
	where 

	\begin{itemize}

		\item $\shrinkAmount_0 = \epsk{\mathrm{init}} \leq \frac{1}{4}$ \\

		\item $\Cconst := 16|V|\max_{(ij)\in E}(\VAL_i\VAL_j) $\\
	
		\item $B = 16 \|\vtheta \|_{1, \infty}$
	
		\item $h_0$ is the initial suboptimality\\

		\item	$k_0$ and $k_1$ are the number of steps to reach stage~2 and~3 respectively which are bounded as: $k_0\leq \max(0,  \lceil\log_{\frac{1}{2}}\frac{\Cconst}{ h_0 \shrinkAmount_0}\rceil )$ 
		$k_0 \leq k_1\leq k_0 + \max\left(0,\lceil \frac{8\Cconst}{\gunifBound\shrinkAmount_0^{2}}\rceil -4\right)$\\
	\end{itemize}
\end{lemma}
\begin{proof}
	Using $\diam_{\|\cdot \|_{\infty,1}}(\MARG) \leq 2$,
	we bound $\Cconst\leq L\diam_{\|\cdot \|_{\infty,1}}(\MARG)^2$ with $L$ (from Property \ref{prop:prop_bounded_grad}) derived in Lemma \ref{lem:trw_bounded_lip_eps}. 
	We bound $-8g_u(\vmu^{(k)})$ (the upper bound on the negative uniform gap) using the value derived in Lemma \ref{lem:bounded_gu0}.
	The rate then follows directly from Theorem \ref{thm:convergence_adaptive_eps} using $p=1$ (see Lemma \ref{lem:trw_bounded_lip_eps} where $L_\shrinkAmount \leq \frac{L}{\delta}$).
\end{proof}
The dominant term in Lemma~\ref{thm:convergence_adaptive_eps_trw} is $\Cconst B \, k^{-\frac{1}{2}}$, with $\Cconst B  = O( \|\vtheta \|_{1, \infty} |V|)$.
We thus find that both bounds depend on norms of $\vtheta$. This is unsurprising since large potentials
drive the solution of the marginal inference problem away from the centre of $\MARG$, corresponding to regions of high entropy,
and towards the boundary of the polytope (lower entropy). Regions of low entropy correspond to smaller components
of the marginal vector, which in turn result in larger and poorly behaved gradients of $-\TRW$, which slows down the resulting optimization.

\section{Correction and Local Search Steps in Algorithm \ref{alg:algInfadaptive}}

Algorithm~\ref{alg:algReopt} details the $\CORRECTION$ procedure used in line~16 of Algorithm~\ref{alg:algInfadaptive} to implement the correction step of the FCFW algorithm.
It uses the modified Frank-Wolfe algorithm (FW with away steps), as detailed in Algorithm~\ref{alg:MFW}.
Algorithm~\ref{alg:algLocalSearch} depicts the $\LOCALSEARCH$ procedure used in line~17 of Algorithm~\ref{alg:algInfadaptive}. The local search is performing FW over $\Meps$ for
a fixed $\shrinkAmount$ using the iterated conditional mode algorithm as an approximate
FW oracle. This enables the finding in a cheap of way of more vertices to augment
the correction polytope $V$.

\begin{algorithm}
	\caption{Re-Optimizing over correction polytope $V$ using MFW, $f$ is the negative TRW objective}
	\label{alg:algReopt}
	\begin{algorithmic}[1]
		\STATE $\CORRECTION(\x^{(0)},V,\shrinkAmount,\vrho)$
		\STATE Let $f(\cdot) := \text{-TRW}(\cdot;\vtheta,\vrho)$; we use MFW to optimize over the contracted correction polytope $\conv(V_\shrinkAmount)$ 
		where $V_\shrinkAmount := (1-\shrinkAmount) V + \shrinkAmount \unif$.
		\STATE Let $\epsilon$ be the desired accuracy of the approximate correction.
		\STATE Let $\bm{\alpha}^{(0)}$ be such that $\x^{(0)} = \sum_{\vv \in V_\shrinkAmount} \alpha^{(0)}_\vv \vv$.
		\STATE $\x^{(\mathrm{new)}} \leftarrow \textbf{MFW}(\x^{(0)}, \bm{\alpha}^{(0)}, V_\shrinkAmount, \stopCrit)$ \algComment{see Algorithm~\ref{alg:MFW}}
		\STATE \textbf{return} $\x^{(\mathrm{new)}}$
	\end{algorithmic}
\end{algorithm}

\begin{algorithm}
	\caption{Local Search using Iterated Conditional Modes, $f$ is the negative TRW objective}
	\label{alg:algLocalSearch}
	\begin{algorithmic}[1]
		\STATE $\LOCALSEARCH(\x^{(0)},\vv_{\mathrm{init}},\shrinkAmount,\vrho)$
		\STATE $\s^{(0)}\leftarrow \vv_{\mathrm{init}}$
		\STATE $V\leftarrow {\emptyset}$
		\FOR{$k=0\ldots \MAXITS$}
		 \STATE $\tilde{\theta} = \nabla f(\xk;\vtheta,\vrho)$
		 \STATE $\snext\leftarrow $ICM$(-\tilde{\theta},\sk)$ \quad\emph{\small (Approximate FW search using ICM; \\ \hspace{35mm} we initialize ICM at previously found vertex $\sk$)}
		 \STATE $\snext_{(\shrinkAmount)} \leftarrow (1-\shrinkAmount)\s^{(k+1)}+\shrinkAmount\unif$
		 \STATE $V\leftarrow V \cup \{\snext\}$
		 \STATE $\dkeps\leftarrow\snext_{(\shrinkAmount)}-\x^{(k)}$
		 \STATE Line-search: $\stepsize_k \in \displaystyle\argmin_{\stepsize \in [0,1]} \textstyle f\left(\x^{(k)} + \stepsize \dkeps\right)$
 		 \STATE Update $\x^{(k+1)} := \x^{(k)} + \stepsize_k  \dkeps$  \algComment{FW update}
		\ENDFOR 
		\STATE \textbf{return} $\xnext,V$
	\end{algorithmic}
\end{algorithm}

\section{Comparison to perturbAndMAP}

\textbf{Perturb \& MAP.} We compared the performance between our method and perturb \& MAP for 
inference on $10$ node Synthetic cliques. We expand on the method we used to evaluate perturbAndMAP in Figure~\ref{fig:syntheticComplete_tightening_l1} 
and~\ref{fig:syntheticComplete_tightening_logz}. 
We re-implemented the algorithm to estimate the partition function in Python (as described in \citet{hazan2012partition}, Section 4.1) 
and used toulbar2~\citep{allouche2010toulbar2}
to perform MAP inference over an inflated graph where every variable maps to
five new variables. 
The log partition function is estimated as the mean energy of 10 exact MAP calls
on the expanded graph where the single node potentials are perturbed by draws from the Gumbel distribution.
To extract marginals, we fix the value of a variable to every assignment, estimate the log partition function 
of the conditioned graph and compute beliefs based on averaging the results of adding the unary potentials
to the conditioned values of the log partition function. 
\section{Correction Steps for Frank-Wolfe over $\MARG$}
Recall that the correction step is done over the correction polytope, 
the set of all vertices of $\MARG$ encountered thus far in the algorithm.
On experiments conducted over $\MARG$, we found that using a better correction algorithm often \emph{hurt} performance.
This potentially arises in other constrained optimization problems
where the gradients are unbounded at the boundaries of the polytope. We found that better correction steps
over the correction polytope (the convex hull of the vertices explored by the MAP solver, denoted $V$ in Algorithm \ref{alg:algInfadaptive}), often resulted in a solution at or near a boundary of the marginal polytope (shared
with the correction polytope).
This resulted in the iterates becoming too small. We know that the Hessian of $\TRW$ is ill conditioned
near the boundaries of the marginal polytope. Therefore, we hypothesize that this is because the 
gradient directions obtained
when the iterates became too small are simply less informative.
Consequently, the optimization over $\MARG$ suffered. We found that the duality gap over $\MARG$ 
would often increase after a correction step when this phenomenon occurred.
The variant of our algorithm based on $\Meps$ is less sensitive to this issue since the restriction 
of the polytope bounds the smallest marginal and therefore also controls the quality of the gradients obtained.

\section{Additional Experiments}

For experiments on the $10$ node synthetic cliques, we can also track the average number of ILP calls required to converge to a fixed 
duality gap
for any $\theta$. This is depicted in Figure~\ref{fig:syntheticILPvstheta}. 
Optimizing over $\TREEPOL$ realized three to four times as many MAP calls as the first iteration of inference. 

Figure~\ref{fig:chineseChar_appendix} depicts additional examples from the Chinese Characters test set. Here, we also visualize results
from a wrapper around TRBPs implementation in libDAI~\citep{Mooij_libDAI_10} that performs tightening over $\TREEPOL$. 
Here too we find few gains over optimizing over $\mathbb{L}$.

Figure~\ref{fig:MvsM_eps_l1_1},~\ref{fig:MvsM_eps_l1_2} depicts the comparison of convergence of algorithm variants 
over $\MARG$ and $\Meps$ (same setup as Figure~\ref{fig:M_M_eps},~\ref{fig:M_M_eps2}. Here, we plot~$\zeta_{\mu}$.

\begin{figure}
\centering
\includegraphics[width=0.8\textwidth]{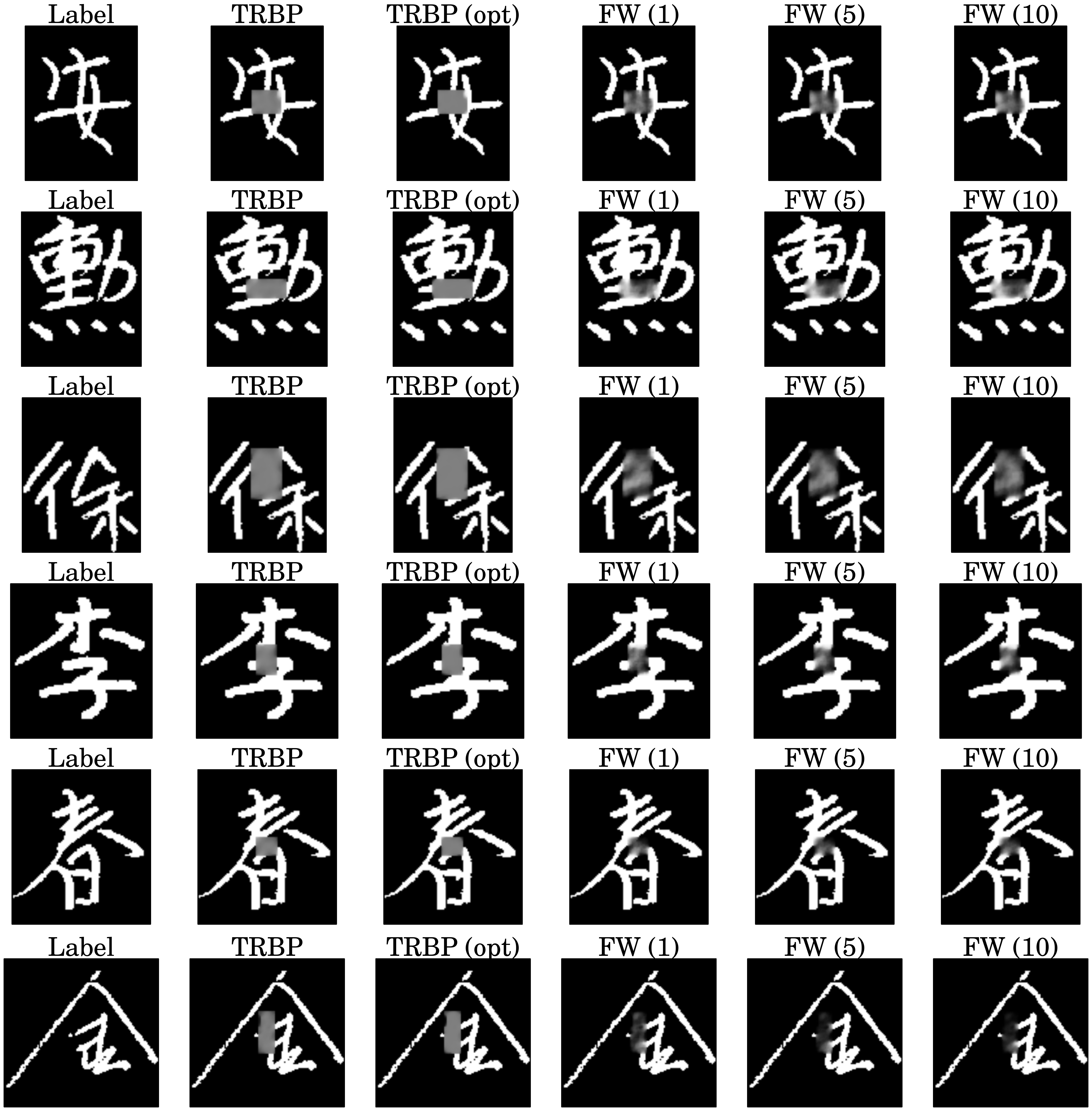}
\caption{Chinese Characters : Additional Experiments. TRBP (opt) denotes our implementation of tightening over $\TREEPOL$ using a wrapper over libDAI~\citep{Mooij_libDAI_10}}
\label{fig:chineseChar_appendix}
\end{figure}

\begin{figure}
\centering
\subfigure[\small{$\zeta_{\mu}$: $5\times5$ grid, $\MARG$ vs $\MARG_{\shrinkAmount}$}]{
	\includegraphics[height=6cm,width=4cm,keepaspectratio]{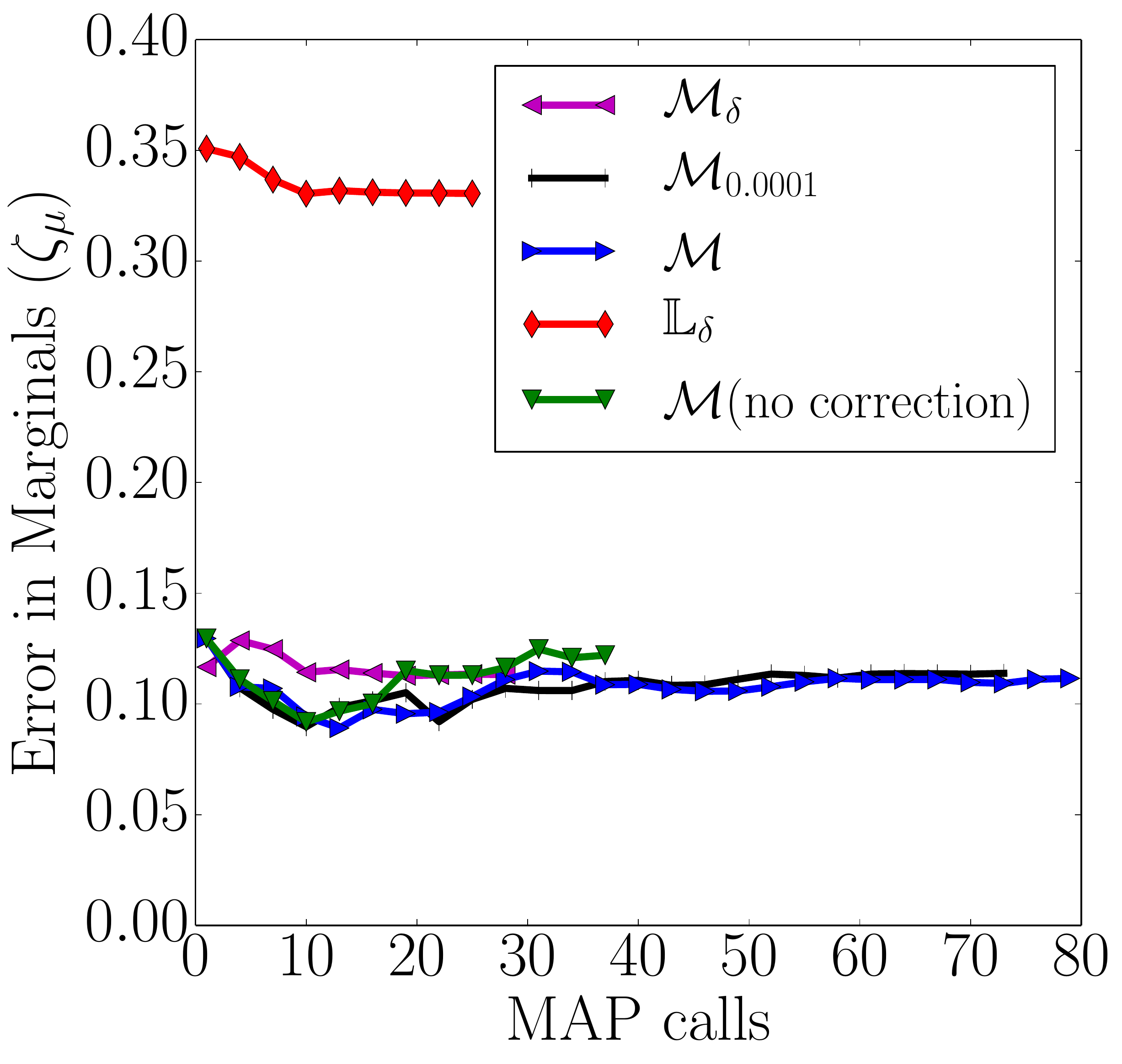}
\label{fig:MvsM_eps_l1_1}
}
\centering
\subfigure[\small{$\zeta_{\mu}$: $10$ node clique, $\MARG$ vs $\MARG_{\shrinkAmount}$}]{
\includegraphics[height=6cm,width=4cm,keepaspectratio]{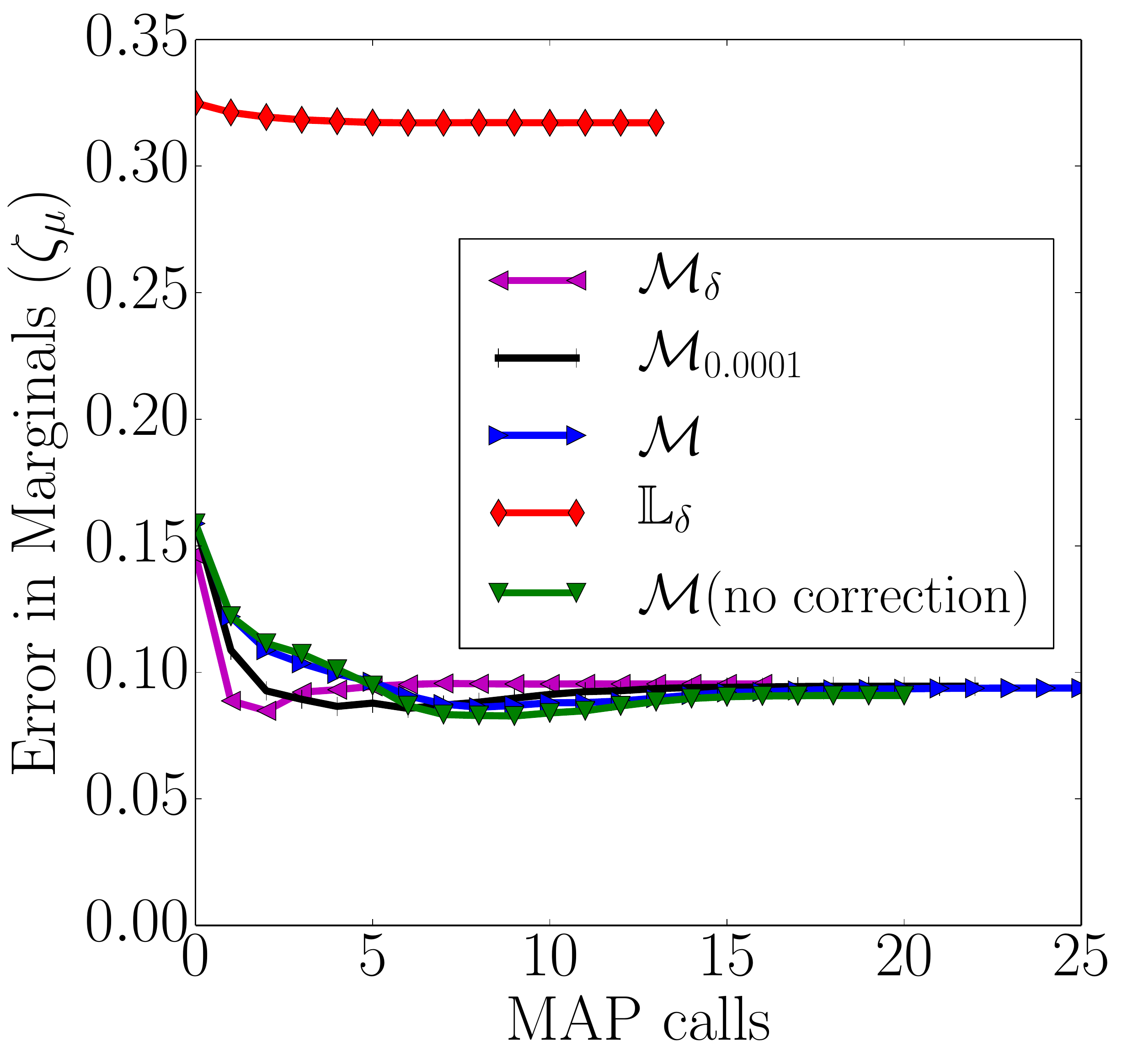}
\label{fig:MvsM_eps_l1_2}
}
\centering
\subfigure[\small{Average ILP calls versus $\theta$: $10$ node clique}]{
\includegraphics[height=4cm,width=5cm,keepaspectratio]{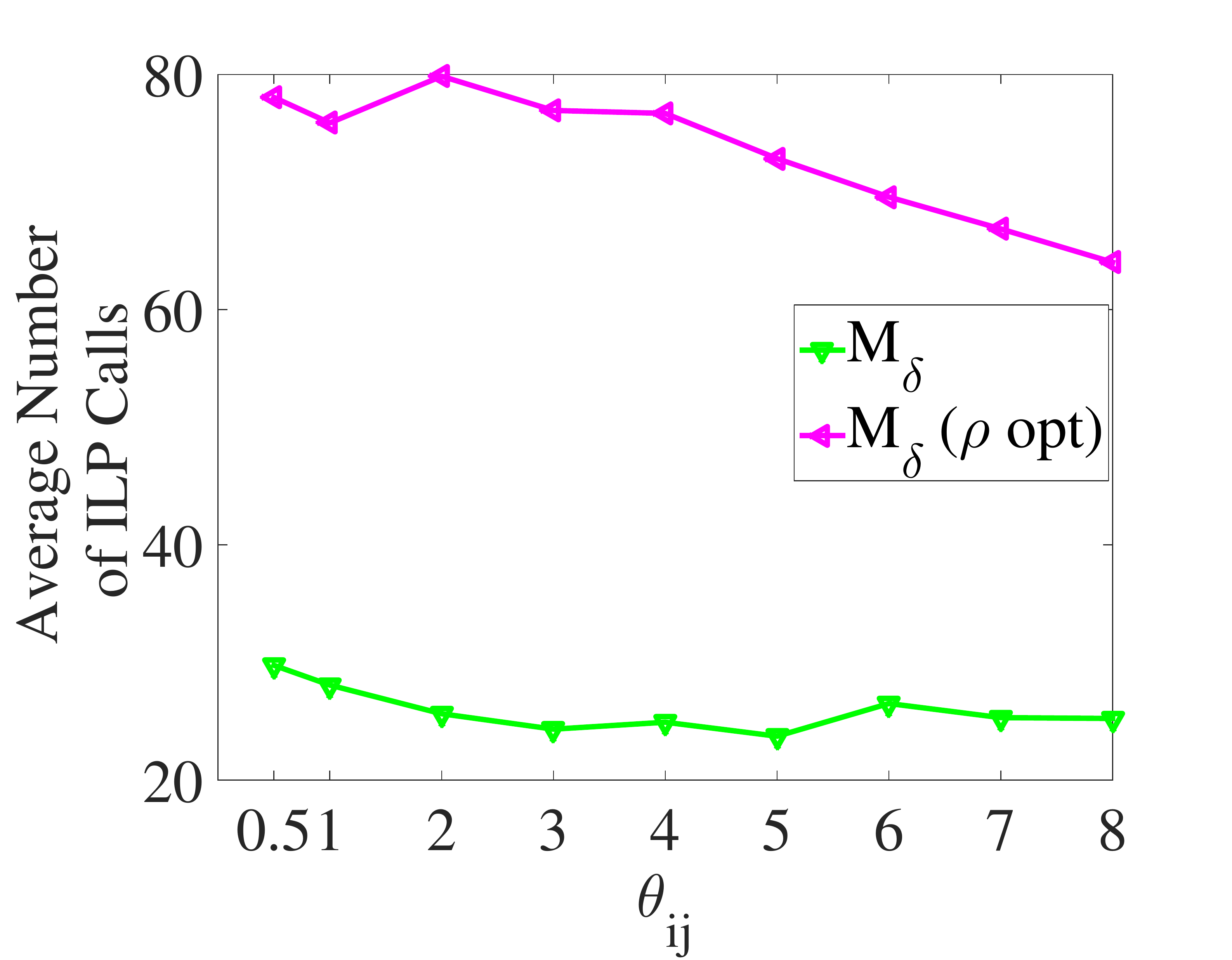}
\label{fig:syntheticILPvstheta}
}
\caption{Figure~\ref{fig:MvsM_eps_l1_1},~\ref{fig:MvsM_eps_l1_2} depict $\zeta_{\mu}$ corresponding to the experimental setup in Figure~\ref{fig:M_M_eps},~\ref{fig:M_M_eps2} respectively. Figure \ref{fig:syntheticILPvstheta} explores the average number of ILP calls taken to convergence with and without optimizing over $\vrho$} 
\vspace{-3mm}
\end{figure}

\section{Bounding $\log Z$ with Approximate MAP Solvers\label{sec:approx_map_logz}}

Suppose that we use an approximate MAP solver for line 7 of Algorithm \ref{alg:algInfadaptive}.
We show in this section that if the solver returns an {\em upper bound} on the value of the MAP
assignment (as do branch-and-cut solvers for integer linear programs), we can use this to get an upper bound on $\log Z$.
For notational consistency, we consider using Algorithm
\ref{alg:algInfadaptive} for $\min_{\x\in\DOM}f(\x)$, where
$f(\x) = -\TRW$ is convex, $\x=\vmu$, and $\DOM = \MARG$.

The property that the duality gap may be used as a certificate of optimality \citep{jaggi2013revisiting} gives us:
\begin{equation}
	\label{eqn:dual_gap_cert}
f(\bm{x}^*)\geq\fk-\gk\implies -f(\bm{x}^*)\leq -\fk + \gk.
\end{equation}
Adding the gap onto the TRW objective yields an upper bound on the optimum (which from Equation \ref{eqn:upperBoundPartition} is an
upper bound on $\log Z$), i.e. $\log Z\leq -f(\bm{x}^*)$. From our definition of the duality gap $\gk$ (line 8 in Algorithm \ref{alg:algInfadaptive}) and \eqref{eqn:dual_gap_cert}, we have:
\begin{align*}
\log Z \leq -f(\bm{x}^*)     &\leq -\fk + \innerProd{-\nabla \fk}{\sk-\xk} \\
&= -\fk + \underbrace{\innerProd{-\nabla \fk}{\sk}}_{\text{MAP call}} - \underbrace{\innerProd{-\nabla \fk}{\xk}}_{\text{Can be computed efficiently}},
\end{align*}
where $\sk = \argmin_{\vv\in\DOM}\innerProd{\nabla \fk}{\vv} =
\argmax_{\vv\in\DOM}\innerProd{-\nabla \fk}{\vv}$ (line~7 in Algorithm~\ref{alg:algInfadaptive}). Thus, if the approximate MAP solver returns
an upper bound $\kappa$ such that $\max_{\vv\in\DOM}\innerProd{-\nabla
  \fk}{\vv} \leq \kappa$, then we get the following upper bound on the
log-partition function:
\begin{equation}
\log Z \leq -\fk + \kappa - \innerProd{-\nabla \fk}{\xk}.
\end{equation}

For example, we could use a linear programming relaxation or a
message-passing algorithm based on dual decomposition such as
\citet{SontagEtAl_uai08} to obtain the upper bound $\kappa$.
There is a subtle but important point to note about this approach.
Despite the fact that we may use a relaxation
of $\MARG$ such as $\LOCAL$ or the cycle relaxation to compute the upper bound, we evaluate it at $\vmu^{(k)}$
that is {\em guaranteed} to be within $\MARG$. This should be
contrasted to instead optimizing over a relaxation such as $\LOCAL$ directly with Algorithm \ref{alg:algInfadaptive}. 
In the latter setting, the moment we move towards a fractional vertex (in line ~14)
we would immediately take $\vmu^{(k+1)}$ out of $\MARG$. Because
of this difference, we expect that this approach will typically result in
significantly tighter upper bounds on $\log Z$.

\bibliographystylesup{abbrvnat}
\small{
\bibliographysup{ref}
}

\end{document}